 \documentclass[accepted]{uai2022} 

\usepackage[american]{babel}

\usepackage{natbib} 
\usepackage{mathtools} 
\usepackage{booktabs} 
\usepackage[noend]{algorithm, algorithmic}
\usepackage{amsmath, amssymb, amsthm}

\newtheorem{theorem}{Theorem}
\newtheorem{corollary}[theorem]{Corollary}
\newtheorem{lemma}[theorem]{Lemma}

\usepackage{tikz} 

\title{ Optimal Exact Matirx Completion Under new Parametrization}

\author{\href{mailto:<iramazan@alumni.cmu.edu>?Subject=Your UAI 2022 paper}{Ilqar Ramazali}{}}
\author{Barnabas Poczos}
\affil{%
    Carnegie Mellon University \\
    
    Pittsburgh, Pennsylvania, USA
}

\begin{document}
\maketitle

\begin{abstract}
We study the problem of exact completion for $m \times n$ sized matrix of rank $r$ with the adaptive sampling method.
 We introduce a relation of the exact completion problem with the sparsest vector of column and row spaces (which we call \textit{sparsity-number} here). 
Using this relation, we propose matrix completion algorithms that exactly recovers the target matrix. 
These algorithms are superior to previous works in two important ways.
First, our algorithms exactly recovers $\mu_0$-coherent column space matrices by probability at  least $1 - \epsilon$ using much smaller observations complexity than - $\mathcal{O}(\mu_0 rn \mathrm{log}\frac{r}{\epsilon})$—the state of art.
Specifically, many of the previous adaptive sampling methods require to observe the entire matrix when the column space is highly coherent.
However, we show that our method is still able to recover this type of matrices by observing a small fraction of entries under many scenarios.
Second, we propose an exact completion algorithm, which requires minimal pre-information as either row or column space is not being highly coherent. 
At the end of the paper, we provide experimental results that illustrate the strength of the algorithms proposed here. \\
\end{abstract}

\section{Introduction}\label{sec:intro}
In this paper, we illustrate how adaptivity helps us to reach theoretical lower bounds concerning observation count in the matrix completion problem.
In modern data analysis, it has been presented that in many scenarios, adaptive sensing and sampling can work more efficiently than passive methods \citep{haupt,kzmn}.
Our main objective is to further optimize adaptive sampling by minimizing the number of observations needed to recover the target matrix entirely.
We show how to recover the entire low-rank matrix by observing information-theoretically least number of entries in various settings. \medskip \\
Low-rank matrix completion plays a significant role in many real-world applications, including camera motion inferring, multi-class learning, positioning of sensors, and gene expression analysis \citep{ nina, akshay1}.
In gene expression analysis, the target matrix represents expression levels for various genes across several conditions.
Measuring gene expression, however, is expensive, and we would like to estimate the target matrix with a few observations as possible. In this paper, we provide an algorithm that can be used for matrix completion from limited data. 
Roughly speaking, to find each unknown expression level, we are supposed to do multiple measurements.
Each of the additional measurements has its extra cost. 
Naturally, we aim to solve the entire problem using the least possible measurement cost. \medskip \\
Krishnamurthy and Singh (\citeyear{akshay1}, \citeyear{akshay2}) illustrated how adaptive sampling reduces observation complexity compared to passive sampling.
These results were two of the earliest algorithms that were robust against coherent row space.
Like many other results in the literature, these algorithms also heavily rely on the incoherence of the column space. 
At first, authors showed for an $n\times n$ sized, rank $r$ matrix, with column space coherence of $\mu_0$, can be exactly recovered using just $\mathcal{O}(n\mu_0 r^{1.5} \log{r})$ observations (\citeyear{akshay1}), then this result optimized to  $\mathcal{O}(n\mu_0 r\log^2{r})$ in the later work (\citeyear{akshay2}).
Recently, \cite{nina} further improved previous results by proposing algorithm that performs $\mathcal{O}(n\mu_0 r \log{r})$ observations to accomplish the task.\medskip \\
The main goal of this paper to give a new approach to the exact recovery problem using the sparsest vector of column and row spaces instead of coherence.
Finding sparsest vector has been in the focus of the research attention for a long time  (\cite{sparse2}, \cite{sparse1}, \cite{sparse3}). 
However, to the best of our knowledge, it is the first time applied to active matrix completion in this paper. \medskip \\
In this particular work, we approach the exact completion problem in various given pre-information settings. 
Our first algorithm requires the precise value of the rank and no other information. Our second algorithm does not request any information except knowing that either column or row space is not highly coherent.
For comparison, as we discuss in the next sections, previous adaptive sampling recovery algorithms require the value of $n r\mu_0  \mathrm{polylog}{\frac{r}{\epsilon}}$, which implicitly requests estimation or exact value of $r$ and $\mu_0$.

\paragraph{Our Contributions.} In light of the above discussion, we state the  main contributions of this work.
\begin{itemize}
\item Relation of the sparsest vector of the column and row space and the problem of exact recovery has been studied in detail. 
An exact completion algorithm is proposed with respect to these vectors.
Moreover, using the relation of the sparsest vector to coherence number, we show that the proposed method exactly recovers the underlying low-rank matrix using less observation than the state of the art.
\item We provide efficient algorithms that require minimal information as $\mathbf{ERR}$ (rank), $\mathbf{ERRE}$ (either row or column space is not coherent).
Moreover, we show the observation complexity of these algorithms is upper bounded by $\mathcal{O}(nr \mu_0 \log{\frac{r}{\epsilon}})$. (the expression for observation complexity
is provided in the next sections after all the necessary definitions are given).
\item  To the best of our knowledge, all previous adaptive sampling methods need to observe entire matrix if the underlying matrix has a highly coherent column space. 
In the algorithm $\mathbf{ERRE}$  we show that having incoherent row space can be a backup and we can still recover these matrices using a small fraction of entries even it has highly coherent column space.

\end{itemize}

\paragraph{Related Work.} 
The power of adaptive sampling had been illustrated even earlier than \cite{akshay1}.
\citep{haupt, ramazanli2022performance, malloy, balak, castro,  ramazanli2020adaptive} showed that under certain hypothesis, adaptive sampling outperforms all passive schemes. \medskip \\
Exact recovery and matrix completion has been studied extensively under passive schemes as well. 
Nuclear norm minimization is one of the most popular methods \cite{gross}.
\cite{recht1} and \cite{recht2} showed that $\Omega ((m + n)r \mathrm{max} (\mu_0^2, \mu_1^2)\log^2{n_2})$ observations are enough to recover an $m \times n$ matrix of rank $r$ using nuclear norm minimization, where $\mu_0$ and $\mu_1$ correspond to column and row space coherence parameters.
Using the same technique, \cite{tao} showed that under the uniform sampling setting we need at least $\Omega(m r \mu_0 \log{n})$ observations to recover the matrix exactly. This result implies the near optimality of
nuclear norm minimization. Another work using nuclear norm minimization is due to \cite{chen}, in which they show how to recover coherent $n\times n$ sized matrix of rank $r$ using just $\mathcal{O}(nr \log^2 r)$ observations.\medskip \\
\cite{tsybk} showed how to use a nuclear norm minimization approach to approximate noisy
low-rank matrices under some global information conditions. 
This work assumes that $\mu_0$ (coherence of column space) is below a given threshold. This result has a similar flavor to ours in that it works even when there is less initial knowledge about the target matrix.
Later, this result was extended to a point where without any assumption on $\mu_0$ the target matrix could be approximated. 
However, the reconstruction error of approximation becomes worse in this case \cite{mai}. There are other approaches for noisy matrix completion which they mainly focus on parameters that describe how much information an observation reveals \cite{kesha,negah}.

\section{Preliminaries}

\paragraph{Notations.}
Let $\mathbf{M}$ denote the underlying $m \times n$ sized rank-$r$  matrix that we target to recover.  
For any positive integer $n$, let $[n]$ represent the set $\{1,2,...,n\}$. 
For any vector $x=(x_1,x_2,...,x_n)$ of size $n$, $\|x\|_{p}$ will denote the $L_{p}$ norm of it.
We call $x_i$ the $i$'th coordinate of $x$. 
For any, $\Omega \subset [n]$ let $x_{\Omega}$ denote the induced subvector of $x$ from coordinates $\Omega$.
For instance, for the vector $x=(1,2,4,8,9)$ and $\Omega = \{1,3\}$, $x_{\Omega}$ represents the vector $(1,4)$. 
For any $\mathbf{R}\subset[m]$, $\mathbf{M}_{\mathbf{R}:}$ stands for an $|\mathbf{R}| \times n$ sized submatrix of $\mathbf{M}$ that rows are restricted by $\mathbf{R}$. 
We define $\mathbf{M}_{:\mathbf{C}}$ in a similar way for restriction with respect to columns.
Intuitively, $\mathbf{M}_{\mathbf{R}:\mathbf{C}}$ defined for $|\mathbf{R}|\times |\mathbf{C}|$ sized submatrix of $\mathbf{M}$ with rows restricted to $\mathbf{R}$ and columns restriced to $\mathbf{C}$.
Moreover, for the special case $\mathbf{M}_{i:}$ stands for $i$-th row and $\mathbf{M}_{:j}$ stands for the $j$'th column.
Similarly, $\mathbf{M}_{i:\mathbf{C}}$ will represent the restriction of the row $i$ by $C$ and $\mathbf{M}_{\mathbf{R}:j}$ represents restriction of the column $j$ by $\mathbf{R}$.

\subsection{Problem Setup} 

\paragraph{Coherence.} One of the critical factors in the matrix completion problem is due to the coherence parameter of the target matrix \citep{zhng1, cnds1}. 
We define the coherence of an $r$-dimensional subspace $\mathbb{U}$ of $\mathbb{R}^n$ in the following way: 
\vspace*{-1mm}
\begin{align*}
\mu(\mathbb{U}) = \frac{n}{r} \underset{1 \leq j \leq n}{\max} || \mathcal{P}_{\mathbb{U}} e_j ||^2,
\vspace*{-3mm}
\end{align*}
where $e_j$ denotes the j-th standard basis element and $\mathcal{P}_{\mathbb{U}}$ represents the orthogonal projection operator onto the subspace $\mathbb{U}$. 
It is easy to see that if $e_j \in \mathbb{U} $ for some $j\in [n]$, then the coherence will attain its maximum value: $\mu(\mathbb{U}) = \frac{n}{r}$. 
We can see that if $\mathbb{U}$ is equally distant from each standard basis vectors, then $\mu(\mathbb{U})$ will be close to $1$, and additionally, it is lower bounded by $1$.
\paragraph{Sparsest vector /  sparsity number.} 
We unify $L_0$ norm of a vector, $L_0$ semi-norm of a matrix, $L_0$ semi-norm of the basis matrix of a subspace under the
name of \textit{nonsparsity-number} for ease of readability.
Furthermore, we call completion of these norms as \textit{sparsity-number}. 
We can see precise definitions below :  \medskip \\
\textit{Nonsparsity-number} of a vector $x\in \mathbb{R}^m$ represented by $\psi(x)$ and defined as $\psi(x) = \|x \|_0$.
\textit{Nonsparsity-number} of an $m\times n$ sized matrix $\mathbf{M}$ and a subspace $\mathbb{U}\subseteq \mathbb{R}^m$ also represented by $\psi$ and defined as:
\begin{align*}
    \psi(\mathbf{M})&=\mathrm{min}\{\psi(x) | x=\mathbf{M}z \text{ and } z\notin \mathrm{null}(\mathbf{M})  \} \\
    \psi(\mathbb{U})&=\mathrm{min}\{\psi(x) | x\in \mathbb{U} \text{ and } x\neq 0 \}
\end{align*}
Similarly, \hypertarget{sparsitynumber}{\textit{sparsity-number}} of vectors, matrices and subspaces all are represented by $\overline{\psi}$ and they are just completion of \textit{nonsparsity-nunmber}.
\begin{align*}
\overline{\psi}(x) &= m-\psi(x) \\
\overline{\psi}(\mathbf{M}) &=  m- \psi(\mathbf{M}) \\
\overline{\psi}(\mathbb{U}) &=  m- \psi(\mathbb{U}) 
\end{align*}  
The space \textit{sparsity-number} for matrices provides a novel and easy way to analyze adaptive matrix completion algorithms.
In many adaptive matrix completion methods, a crucial step is to decide whether a column is (or is not) contained in a given subspace. 
Ideally, we would like to make this decision as soon as possible before observing the entire column vector.
In this paper, we show that \textit{sparsity-number} helps us to decide whether a partially observed vector can be contained in a given subspace or not. 
In lemma~\ref{lm1}, we provide an answer to this question simply.

\subsection{Matrix Completion with Adaptive Sampling} 
We want to present an algorithm due to \citet{akshay1} here before providing our main results in the next section.
Authors  proposed an adaptive algorithm that can recover $n\times n$ sized  rank-$r$ matrices using $\mathcal{O}(n\mu_0 r^{1.5}\log{r})$ observations, which was indeed better than known state of the art $\mathcal{O}(n\mu_0 r^2 \log^2{n})$ for passive algorithms (\citeyear{recht2}).
The algorithm studies column space by deciding whether the partially observed column is linearly independent with previously fully observed columns.
Authors show observing $\mathcal{O}(\mu_0 r^{1.5}\log{\frac{r}{\epsilon}})$ observations for each column is enough to make the decision for linear independence with probability $1-\epsilon$.
The algorithm below describes the details of the proposed algorithm. 
\begin{algorithm}
\caption*{  \hypertarget{ks2013}{\textbf{KS2013:}} Exact recovery  \citet{akshay1}.}
\textbf{Input:}   $d=\mathcal{O}(\mu_0 r^{1.5}\log{\frac{r}{\epsilon}})$\\
 \textbf{Initialize:}  $k=0 , \widehat{\mathbf{U}}^0 = \emptyset$ 
\begin{algorithmic}[1]
    \STATE Draw uniformly random entries $\Omega \subset [m]$ of size $d$   
    \FOR{$i$ from $1$ to $n$}        
    \STATE  \hspace{0.2in} \textbf{if} $\| \mathbf{M}_{\Omega:i}-{\mathcal{P}_\mathbf{\widehat{U}_{\Omega}^k}} \mathbf{M}_{\Omega:i}\| >0$ 
    \STATE \hspace{0.2in}  \hspace{0.2in} Fully observe $\mathbf{M}_{:i}$ 
    \STATE  \hspace{0.2in}  \hspace{0.2in}  $\widehat{\mathbf{U}}^{k+1} \leftarrow \widehat{\mathbf{U}}^{k} \cup \mathbf{M_{:i}} $, Orthogonalize $\widehat{\mathbf{U}}^{k+1}$,  $k=k+1$          
   \STATE   \hspace{0.2in}  \textbf{otherwise:} $\widehat{\mathbf{M}}_{:i} = \widehat{\mathbf{U}}^k {\widehat{\mathbf{U}}^{k^+}_{\Omega}}
 \widehat{\mathbf{M}}_{\Omega :i}$
\ENDFOR
\end{algorithmic}
\textbf{Output:}  $\widehat{\mathbf{M}}$
\end{algorithm}

Later, authors improved the observation complexity to $\mathcal{O}(n\mu_0 r\log^2{\frac{r}{\epsilon}})$ in a proceeding work (\citeyear{akshay2}).
Then, another improvement due to \citet{nina} further reduced this complexity to $\mathcal{O}(n\mu_0 r\log{\frac{r}{\epsilon}})$-current state of the art using similar setting and algorithm.

\section{Main Results}

In this section, we provide the theoretical results of this paper. 
First, we build the bridge between space \textit{sparsity-number} and the exact completion problem. Then, we show using this connection its possible to give efficient algorithms.

\subsection{Optimal Observation for Each Column}

Before proceeding to more advanced algorithms, we target to answer one fundamental question.
What is the specific number of entries in \hyperlink{ks2013}{$\mathbf{KS2013}$} to observe in a column that allows us to deterministically decide whether it is independent or dependent on previous columns? 
In what follows, we prove \textit{sparsity-number} of the column space answers this question :
\begin{lemma} \label{lm1} Let $\mathbb{U}$ be a subspace of $\mathbb{R}^{m}$ and $x^1,x^2,...,x^n$ be any set of vectors from $\mathbb{U}$. Then the linear 
dependence of $x^1_\Omega, x^2_\Omega,...,x^n_\Omega$ implies linear dependence of 
$x^1,x^2,...,x^n$, for any $\Omega \subset [m]$ such that $|\Omega| > \overline{\psi}(\mathbb{U})$.
\end{lemma}
Merging the idea of the algorithm \hyperlink{ks2013}{$\mathbf{KS2013}$} with lemma 1 we get an algorithm that is presented and discussed in the Supplementary Materials, whose can recover low-rank matrices deterministically ($\mathbf{ERCS}$).
Shortly, setting $d = \psi(\mathbf{M})+1$ is enough to ensure the underlying matrix will be recovered with probability $1$.
Moreover, we analyze the following properties of the algorithm :
\begin{itemize}
\item  Lemma 1 is tight. Setting $d=\psi(\mathbb{U})$ would fail statement
of the lemma for some matrices.
\item  If $\overline{\psi}(\mathbb{U})+1 = r$ satisfied, then the observation complexity of the algorithm is equal to $(n+m-r)r$ which is degree of freedom of rank-$r$ matrices.
\end{itemize}
Moreover, we suggest an exact completion problem where each entry has heterogeneous observation cost and proposed an efficient solution to that given that \textit{sparsity number} of row space is known.\\[1.5ex]
Due to the challenge of finding space \textit{sparsity-number} of the column/row space of the underlying matrix, we didn't provide the algorithm and analysis in the main text.
However, having any lower bound to the \textit{sparsity-numbers} would also be enough successful termination of the algorithm.

\subsection{Exact Completion Problem}

Here, we provide an algorithm that exactly recovers a target matrix under the active setting.
As we discussed before, one of the strengths of the results of  \citet{akshay1,akshay2, nina} is that the algorithm is its robustness to highly coherent row space compared to previous results as \citet{recht2}.\\[1.5ex]
However, we notice that these algorithms are independent of row space and treat any row space  equally.
This phenomenon arises a natural question, whether there is an algorithm which enjoys properties of row space to optimize these algorithms further.\\[1.5ex]
For example, we can see the following matrices having the same rank $r=2$ and column space coherence $\mu_0=2$. The only difference is due to the coherence of row space, which is $3$ for $\mathbf{A}$ and near to $1$ for matrix $\mathbf{B}$.
Similarly, row space \textit{sparsity-number} is $1$ for the matrix $\textbf{A}$ and $4$ for the matrix $\textbf{B}$.\\
\vspace{-4mm}
\begin{align*}
\mathbf{A} = \begin{bmatrix}
    1   & 2 & 2 & 2 & 2 & 2\\
    0   & 2 & 2 & 2 & 2 & 2\\
    0   & 2 & 2 & 2 & 2 & 2\\
    0   & 2 & 2 & 2 & 2 & 2\\
\end{bmatrix} 
\end{align*}

\vspace{-5mm}

\begin{align*}
 \mathbf{B} = \begin{bmatrix}
    1   & 0 & 1 & 2 & 3 & 4 \\
    0   & 1 & 2 & 3 & 4 & 5 \\
    0   & 1 & 2 & 3 & 4 & 5 \\
    0   & 1 & 2 & 3 & 4 & 5  \\
\end{bmatrix} \\
\end{align*}

As previous methods are mainly based on the value of the size of the matrix, $r$, and $\mu_0$, these methods treat both these matrices equally.
Indeed the first column of the matrix $\mathbf{A}$ is crucial to study the column space. 
That's why we don't want to take the risk of missing the necessary information on this column.\\[1.5ex]
Therefore we end up observing many entries in each column to make sure we will not miss this column.
However, it is entirely different for the matrix $\mathbf{B}$; any missed column can be replaced by any other one in the study of the column space. 
That's why it should give us the flexibility of observing less number of entries in each column.

\subsubsection{\textbf{\textit{Exact completion for rank r matrices}}}

In the following algorithm, we propose a method that exactly recovers the  $m\times n$ sized rank $r$ underlying matrix, using just exact information of rank $r$.
The idea of the algorithm is to find $r$-many linearly independent rows and columns and recover the remaining entries based on them.
Finding the linearly independent columns is rely on simple linear algebra fact that if columns of $\mathbf{M}_{R:C}$ are linearly independent, then so are columns of $\mathbf{M}_{:C}$.
The statement is valid for rows as well symmetrically.\\[1.5ex]
Then, all we need to do is to wait for detecting $r$ many independent columns and rows.
Indeed the algorithm does not require an estimate or exact information of coherence as opposed to \hyperlink{ks2014}{$\mathbf{KS2013}$}, \citeyear{nina,akshay2} (coherence is crucial to compute $d$ in input phase).
The improvement for observation complexity can be observed in the following theorem and corollary.
\begin{algorithm}
\caption*{  \hypertarget{err}{\textbf{ERR:}}  Exact recovery for rank $r$ matrices.}

 \textbf{Input:}  Rank of the target matrix - $r$\\
 \textbf{Initialize:} $R = \emptyset$, $C=\emptyset$, $\widehat{r} = 0$
\begin{algorithmic}[1]\label{alg1}
 
    \WHILE{$\widehat{r} < r$}
    \FOR{$j$ from $1$ to $n$}
    \STATE Uniformly pick an unobserved entry $i$ from $\mathbf{M}_{:j}$  
    \STATE $\widehat{R}={R} \cup \{i\} , \widehat{C}= {C} \cup \{j\}$  
    
    \STATE \textbf{If} $ \mathbf{M}_{\widehat{R}:\widehat{C}}$ is nonsingular
    \STATE \hspace{0.1in} Fully observe $\mathbf{M}_{:j}$ and $\mathbf{M}_{i:} $ and set
     $R =\widehat{R}$ , $C = \widehat{C}$ , $\widehat{r}=\widehat{r}+1$
\ENDFOR
\ENDWHILE

\STATE Orthogonalize column vectors in $C$ and assign to $\widehat{\mathbf{U}}$

\FOR{each column $j \in [n]\setminus C$}
\STATE $\widehat{\mathbf{M}}_{:j} = \widehat{\mathbf{U}} {\widehat{\mathbf{U}}_{R:}}^+
 \widehat{\mathbf{M}}_{R: j}$
\ENDFOR
\end{algorithmic}
\textbf{Output:} $\widehat{\mathbf{M}}$
\end{algorithm}
\begin{theorem}\label{thm:lg2} 
Let $r$ be the rank of underlying $m\times n$ sized matrix $\mathbf{M}$ with column space $\mathbb{U}$ and row space $\mathbb{V}$. 
Then,  \hyperlink{err}{$\mathbf{ERR}$} exactly recovers the underlying matrix $\mathbf{M}$ with probability at least $1-\epsilon$  using number of observations at most: 
$$ (m+n-r)r + 
  \mathrm{min}\Big( 2 \frac{m n}{{\psi}(\mathbb{U})} \log{(\frac{r}{\epsilon})},  \frac{\frac{2m}{\psi(\mathbb{U})}(r+2 +\log{\frac{1}{\epsilon}})}{\psi(\mathbb{V})}n \Big)  $$
\end{theorem}

In the appendix section we are introducing the relation between $\psi(U)$ which using it we can imply the following conclusion is always right.

\begin{corollary} \label{cor:err} 
Observation complexity of the algorithm $\mathbf{ERR}$ is always guaranteed to be upper bounded by 
$$(m+n-r)r + \mathcal{O}\big(nr\mu_0\log{(\frac{r}{\epsilon})}\big)$$
and in many cases it is as low as $\mathcal{O}((m+n-r)r )$ which is absolute lower bound for any algorithm. 
\end{corollary}

\subsubsection{\textbf{\textit{Exact Recovery While Rank Estimation}}}

In this section, we solve the exact completion problem in a slightly different setup.
$\mathbf{ERR}$ assumes that we know the exact rank of the underlying matrix.
However, here we assume that we don't have this information. Therefore we don't know precisely at which point the process of searching a new independent row/column should stop.
In what follows, we show that if at a given state new independent column/row not detected for long enough time, then it means that it is likely no more one exists.
We formalize this statement in the following theorem.
\begin{algorithm}
\caption*{  \hypertarget{erre}{\textbf{ERRE:}}  Exact recovery while rank estimation.}

 \textbf{Input:}  $T$-delay parameter at the end of algorithm \\
 \textbf{Initialize:} $R = \emptyset, C=\emptyset,  \widehat{r}= 0,  delay = 0$

\begin{algorithmic}[1]\label{alg1}

    \WHILE{$delay <T $}
    \STATE   $ delay= delay + 1$

    \FOR{$j$ from $1$ to $n$}
    \STATE Uniformly pick an unobserved entry $i$ from $\mathbf{M}_{:j}$  
    \STATE $\widehat{R}={R} \cup \{i\} , \widehat{C}= {C} \cup \{j\}$  
    
    \STATE \textbf{If} $ \mathbf{M}_{\widehat{R}:\widehat{C}}$ is nonsingular :
    \STATE \hspace{0.1in} Fully observe $\mathbf{M}_{:j}$ and $\mathbf{M}_{i:} $ 
and set  $R =\widehat{R}$ , $C = \widehat{C}$ , $\widehat{r}=\widehat{r}+1, delay=0$
    \ENDFOR

\ENDWHILE

\STATE Orthogonalize column vectors in $C$ and assign to $\mathbf{U}$

\FOR{each column $j \in [n]\setminus C$ }
\STATE $\widehat{\mathbf{M}}_{:j} = \widehat{\mathbf{U}} {\widehat{\mathbf{U}}_{R:}}^+ \widehat{\mathbf{M}}_{R: j}$
\ENDFOR

\end{algorithmic}
\textbf{Output:} $\widehat{\mathbf{M}}, \widehat{r}$
\end{algorithm}
\vspace{-1mm}
\begin{theorem}  \label{thm:erre}
Let $r$ be the rank of underlying $m\times n$ sized matrix $\mathbf{M}$ with column space $\mathbb{U}$ and row space $\mathbb{V}$. 
Then, \hyperlink{erre}{$\mathbf{ERRE}$} exactly recovers the underlying matrix $\mathbf{M}$ while estimating rank with probability at least $1- ( \epsilon + e^{-T\frac{\psi(\mathbb{U})\psi(\mathbb{V})}{m}})$  using number of observations at most:
\begin{align*}
  &(m+n-r)r+Tn +  \\[0.8ex]
  \mathrm{min} \Big(  2 \frac{m n}{\psi(\mathbb{U})}&\log{(\frac{r}{\epsilon})} , \frac{\frac{2m}{\psi(\mathbb{U})}(r+2 +\log{\frac{1}{\epsilon}})}{\psi(\mathbb{V})}n ) \Big)
\end{align*}
\end{theorem}
which implies the following corollary right, similar to previous case.
\begin{corollary}\label{cor:erre} Lets assume that either $\psi(\mathbb{U})$ or $\psi(\mathbb{V})$ is big enough i.e. $\psi(\mathbb{U})\psi(\mathbb{V}) \geq m$. {For} a given $\epsilon$  set $T=\log{\frac{1}{\epsilon}}$, then $\mathbf{ERRE}$ recovers underlying matrix with probability $1-2\epsilon$  using just 
\begin{align*}
  &(m+n-r)r+ n\log{\frac{1}{\epsilon}} +\\[0.8ex]
  \mathrm{min} \Big(  2 &\frac{m n}{\psi(\mathbb{U})}\log{(\frac{r}{\epsilon})} , \frac{\frac{2m}{\psi(\mathbb{U})}(r+2 +\log{\frac{1}{\epsilon}})}{\psi(\mathbb{V})}n  \Big)     
\end{align*}
\end{corollary}
We refer to the analysis of corollary \ref{cor:erre} in the Supplementary Material to understand this expression.

\section{Experimental Results}

\textbf{Experiments for Synthetic Data:} We tested all of our exact completion algorithms-\hyperlink{err}{$\mathbf{ERR}$} and \hyperlink{erre}{$\mathbf{ERRE}$}
for synthetically generated low-rank matrices.\\[1.5ex]
To design an $m\times n$ sized rank $r$ with \textit{space sparsity number} equal to $1$, we generate $m\times {r-1}$ and ${r-1}\times n$ sized matrices $\mathbf{X,Y}$, where $X_{i,j},Y_{i,j}\sim \mathcal{N}(0,1)$. 
As we discussed before, multiplication of these matrices would gives us a rank $r-1$ matrix.
Moreover, as the column space of $\mathbf{M}$ is column space of $\mathbf{Y}$.  Given that  $Y_{i,j} \sim \mathcal{N}(0,1)$ implies that coherence of column space of $\mathbf{Y}$ is small. Therefore, $\mathbf{M}$ has small column space coherence number.
We can use the similar argument to claim the coherence of column space is also small as row space of $\mathbf{M}$ is the row space of $\mathbf{X}$. \medskip\\
In order to make column space of this matrix highly coherent, we follow the following procedure:

\begin{itemize}
    \item[1.]   generate a random vector $u\in \mathbb{R}^n$
    \item[2.]   randomly select a number $i \in [m]$ 
    \item[3.]   replace  $\mathbf{M}_{i:}$ with $u$
\end{itemize}
We can guarantee that the resulting matrix will contain $i$-th standard basis vector $e_i$ in the column space. 
To observe this phenomenon, lets analyse the restriction of the matrix $\mathbf{M}$ to the first $r$ columns and all the rows but row $i$.
As this is a submatrix of rank $r-1$ matrix (initial $\mathbf{M}$) this matrix also has rank at most $r-1$.
Therefore, we have non-trivial coefficients $\alpha_1, \alpha_2, \ldots, \alpha_r$, that makes linear combination of columns of submatrix to be equal to zero vector.
Therefore,
\begin{align*}
\alpha_1 \mathbf{M}_{:1}+\alpha_1 \mathbf{M}_{:2}+\ldots+ \alpha_r\mathbf{M}_{:r} = w e_i
\end{align*}
for some $w$. Then all we need to show is $w \neq 0$, however it is straightforward because $w = \alpha_1 u_1+\alpha_2 u_2 +\ldots + \alpha_r u_r$ which is nonzero because $u$ is random. 
In conclusion, $e_i$ is contained in the column space of $\mathbf{M}$ and therefore from the definition of coherence of the columns space of the matrix is equal to $\frac{m}{r}$ which is also maximum value.\medskip\\
We note that, we might change the process above and generate several random vectors and replace them with some rows of the matrix, we would still get high coherence values as $\Omega(\frac{m}{r})$.
Using similar idea we can add replace columns with random vector to get highly coherent row spaced matrices.
Moreover, we can apply both at the same time, to get highly coherent column and highly coherent row space.\\[1.5ex]
We test our methods for fixed $r=5$ and $r=10$ with $n$ varying over $\{1000,2000,\ldots, 12000\}$.
Similarly, we fixed $n=3000$ and varied $r$ over $\{1,2,\ldots, 12\}$.
Moreover, we classified matrices into four types for coherence/incoherence of the column and row space. 
From the table below, we can get the comparison of the performance of our algorithms to previous methods (\cite{nina,akshay1,akshay2,chen}). 

\newpage

\begin{table}[h]
\begin{center}
\begin{tabular}{cc} 
 \includegraphics[height=2.74cm,width=4.2cm]{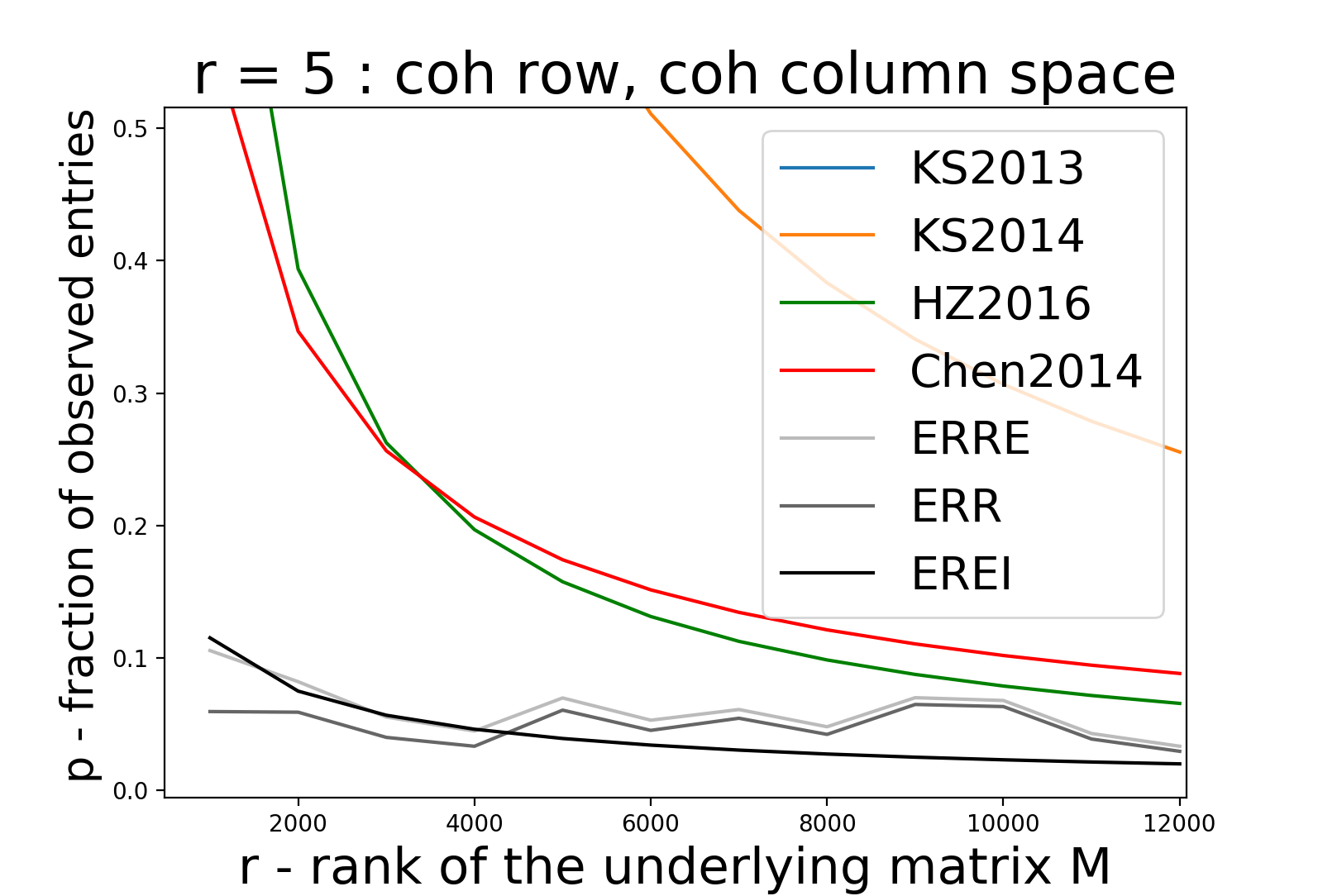} 
 \includegraphics[height=2.74cm,width=4.2cm]{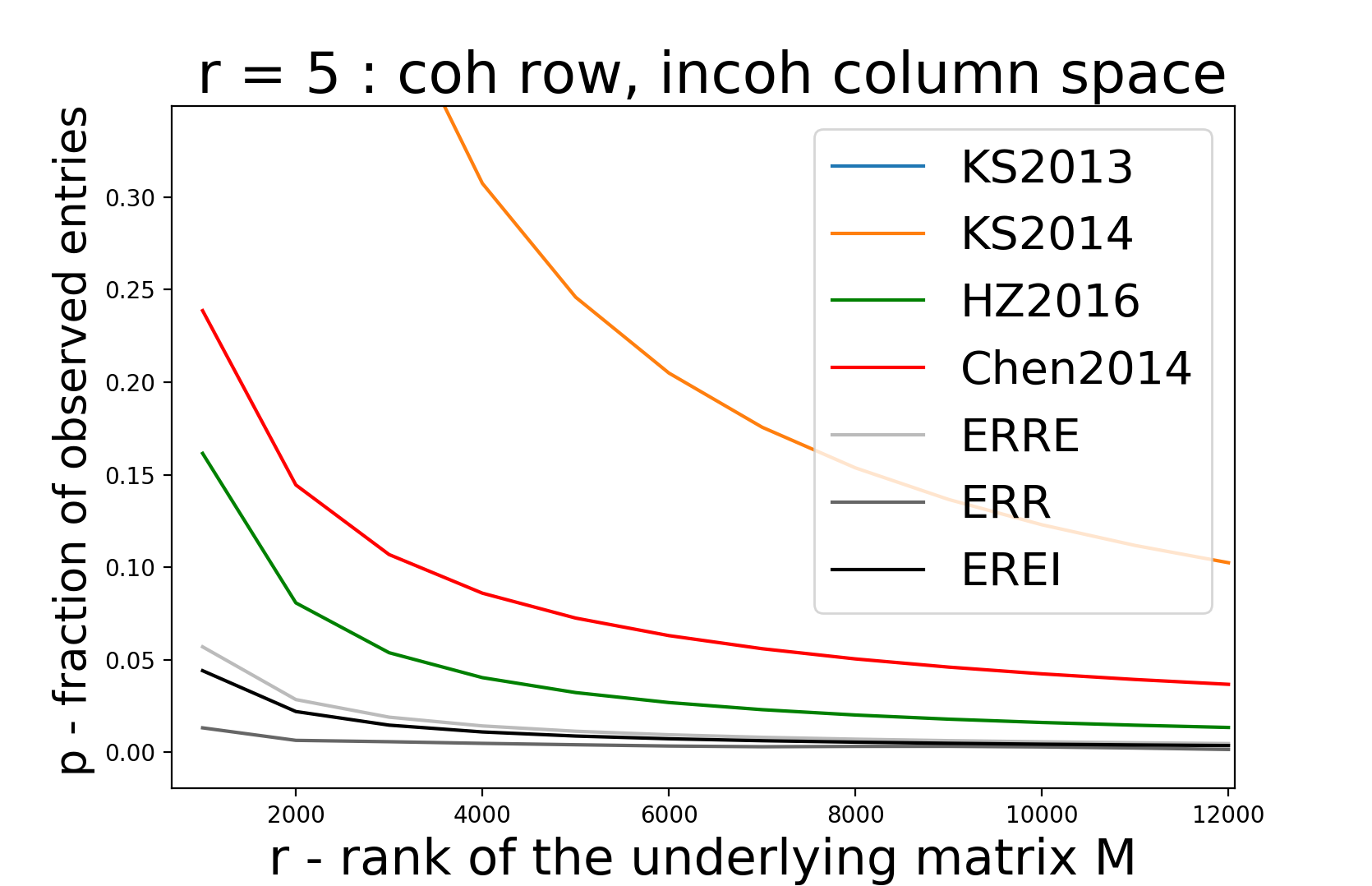} \\
 \includegraphics[height=2.74cm,width=4.2cm]{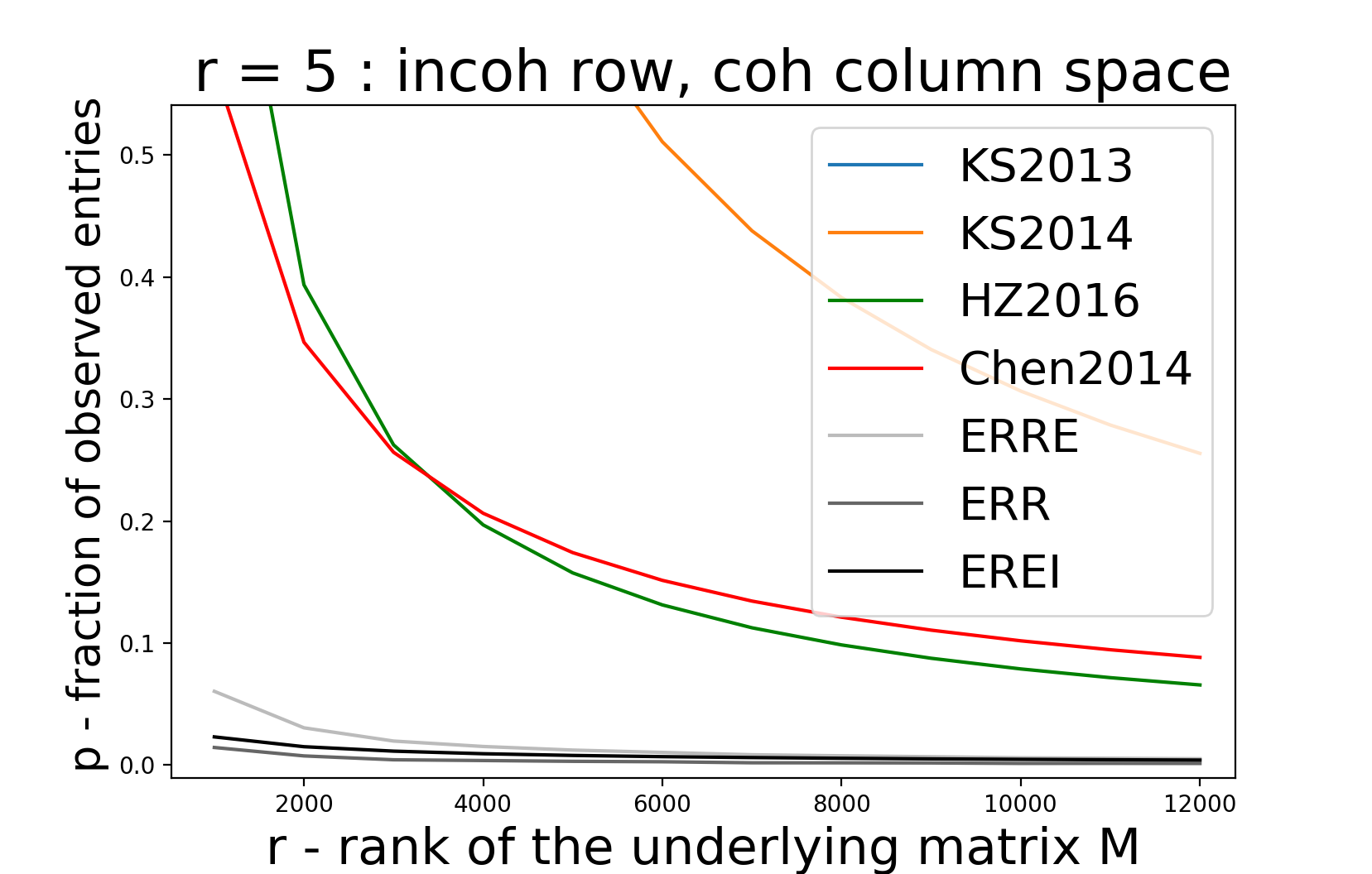} 
 \includegraphics[height=2.74cm,width=4.2cm]{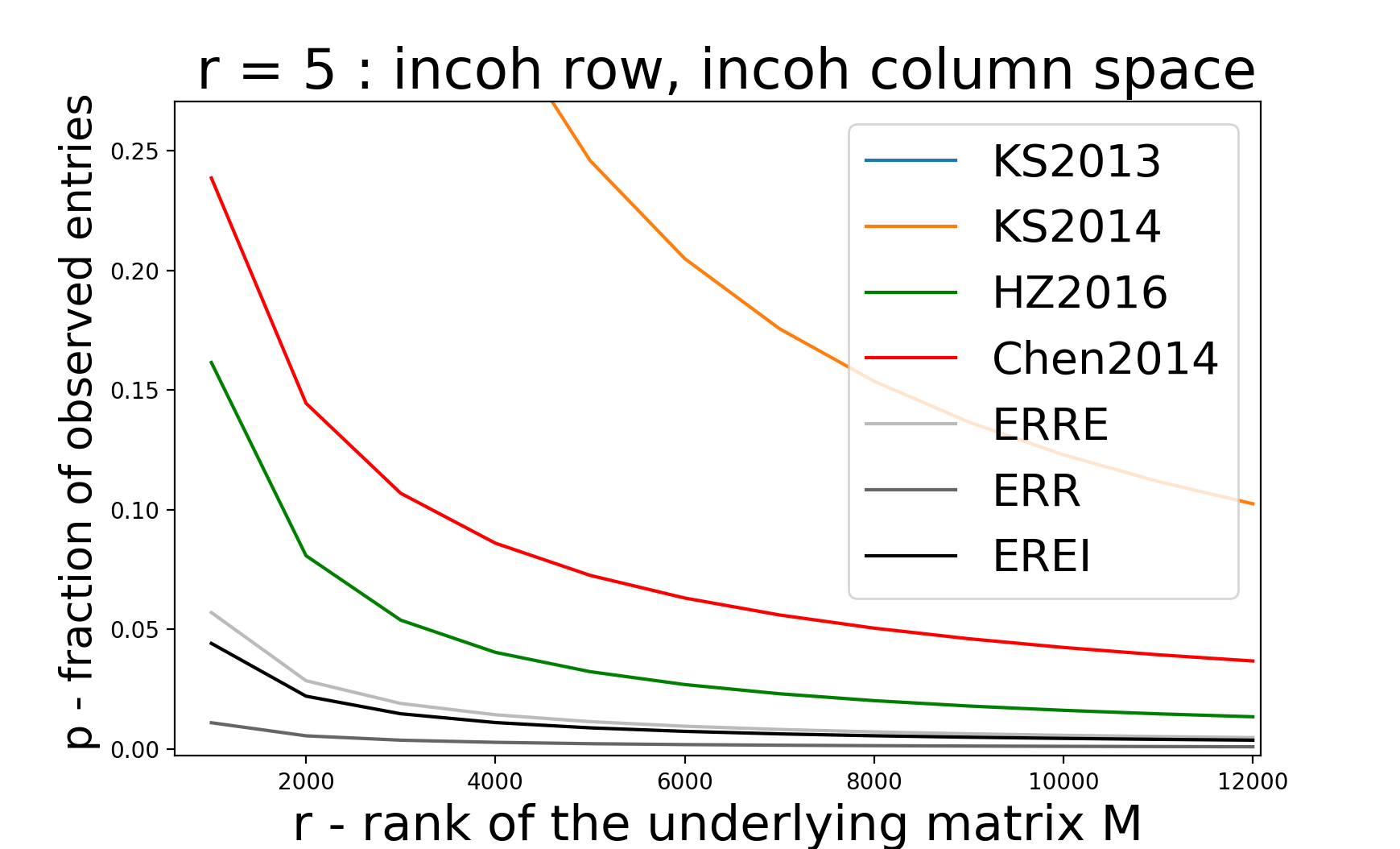}  \\[0.5ex]
 \includegraphics[height=2.74cm,width=4.2cm]{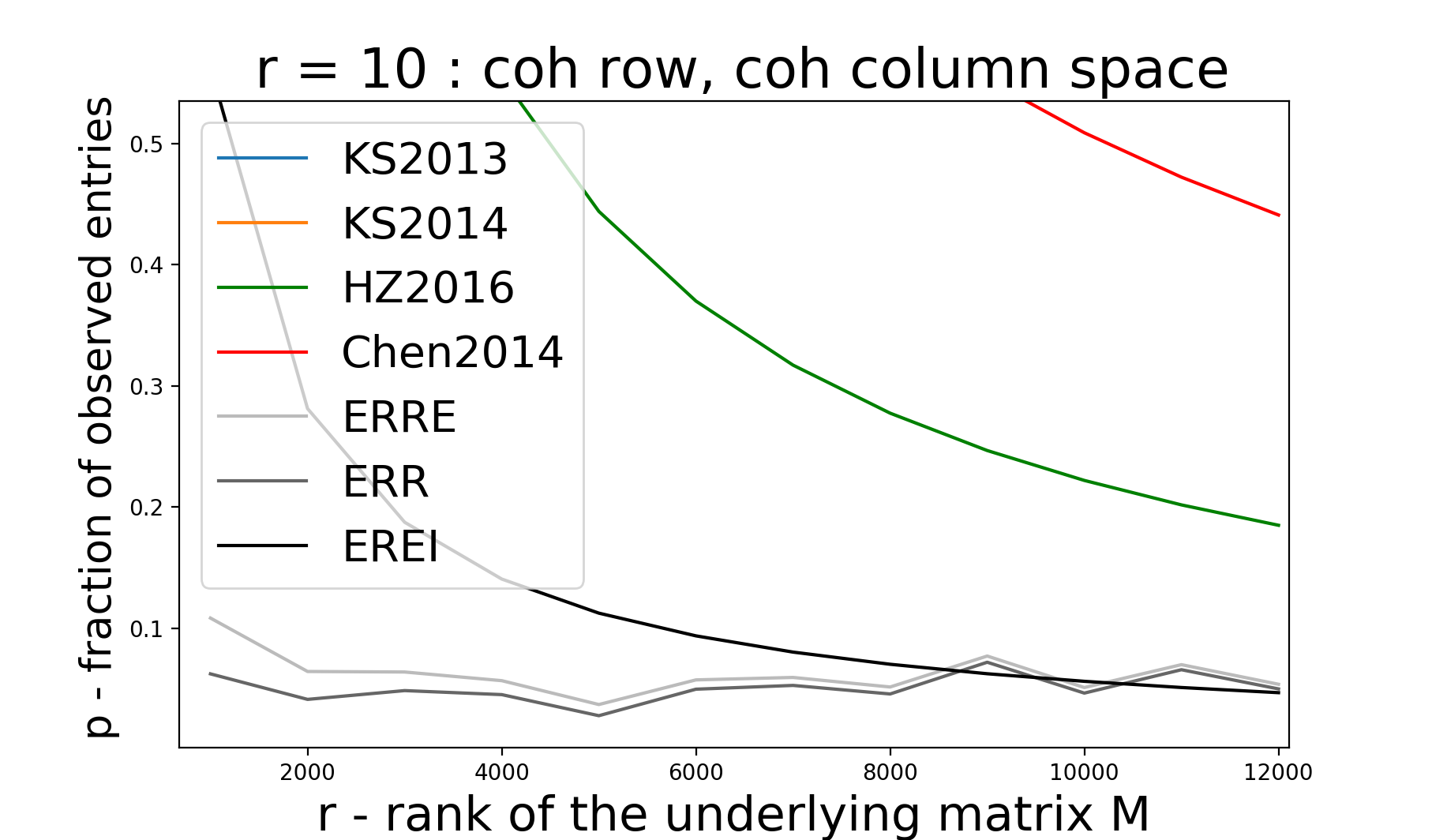} 
 \includegraphics[height=2.74cm,width=4.2cm]{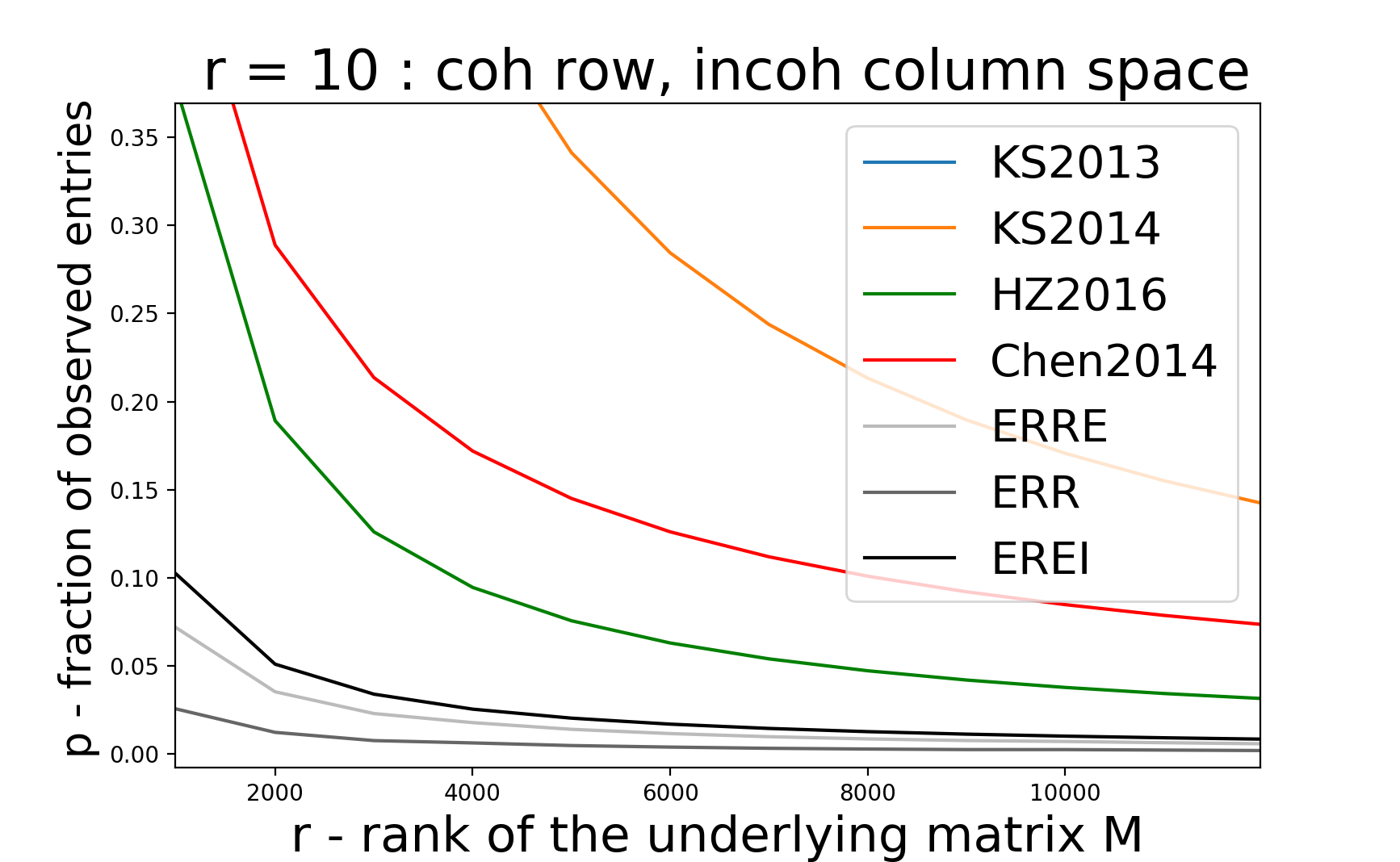} \\
 \includegraphics[height=2.74cm,width=4.2cm]{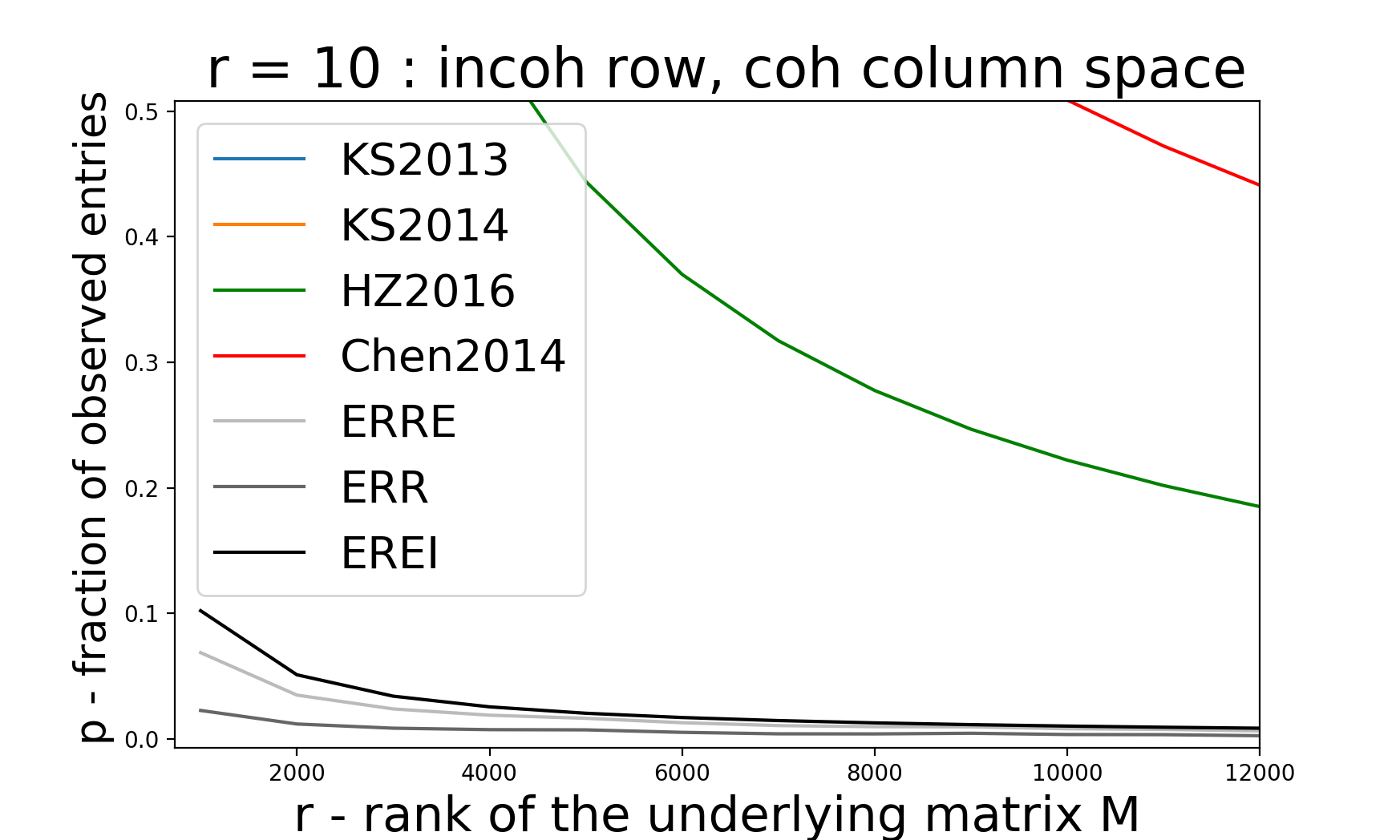} 
 \includegraphics[height=2.74cm,width=4.2cm]{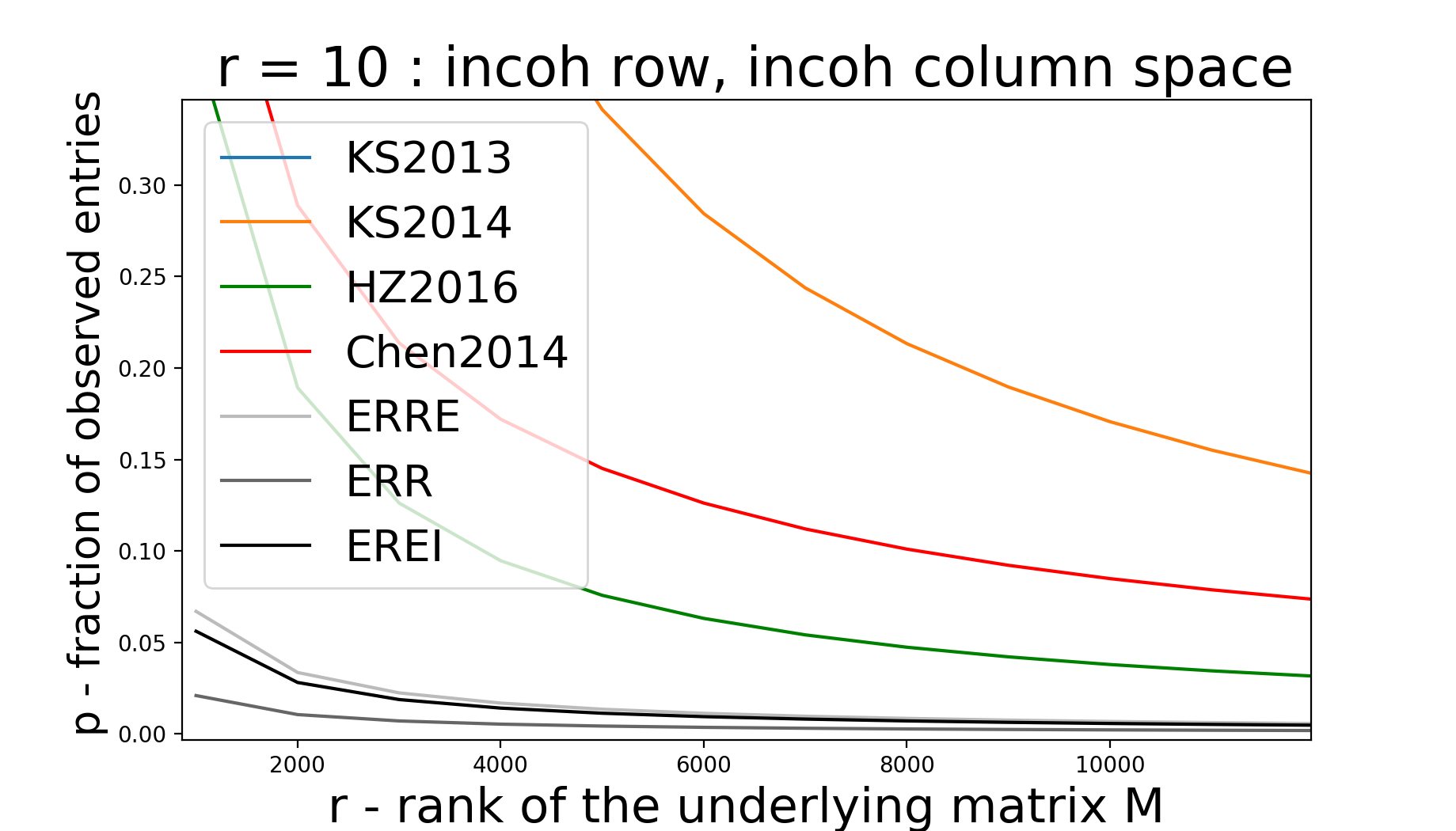}  \\[0.5ex]
\includegraphics[height=2.74cm,width=4.2cm]{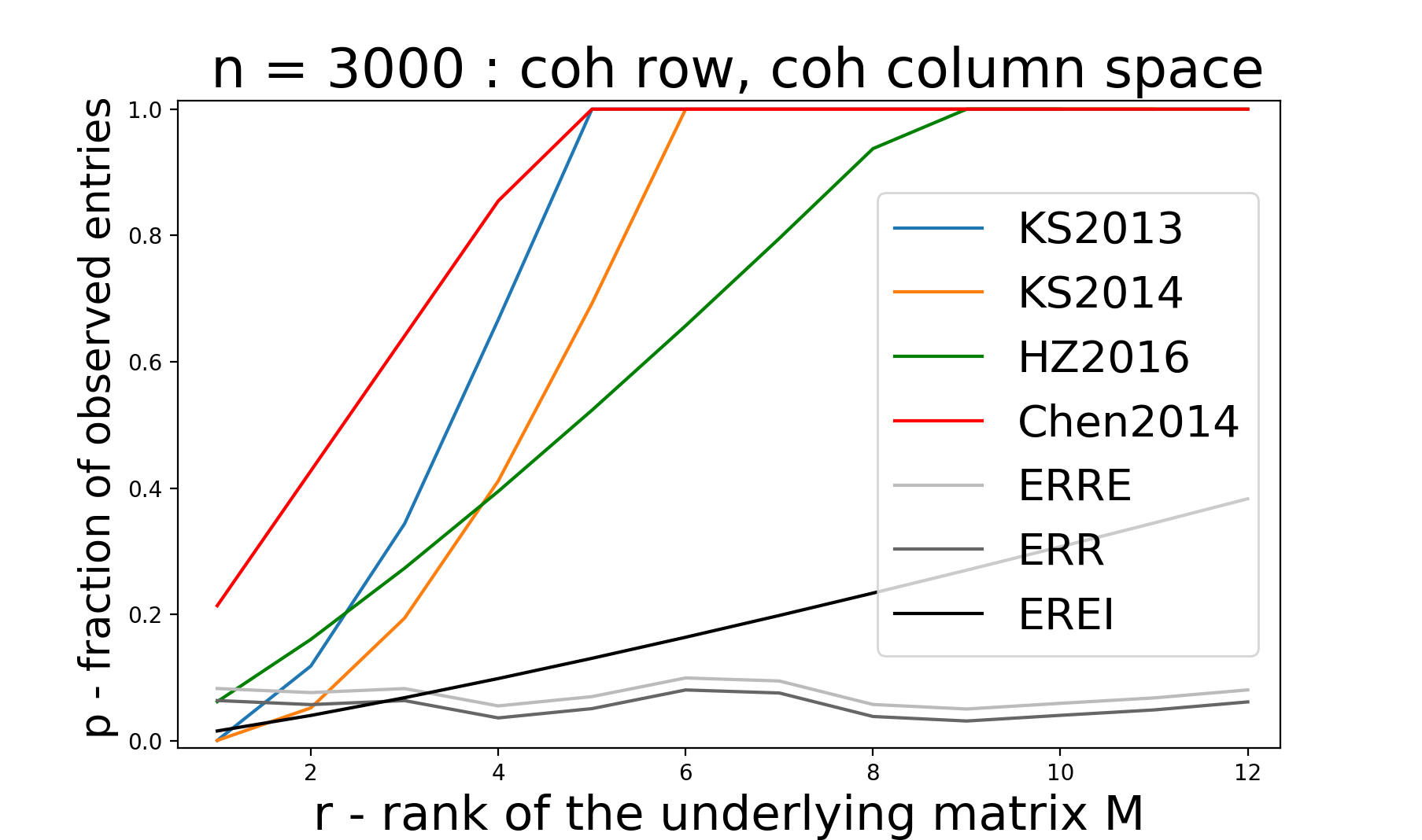}     
\includegraphics[height=2.74cm,width=4.2cm]{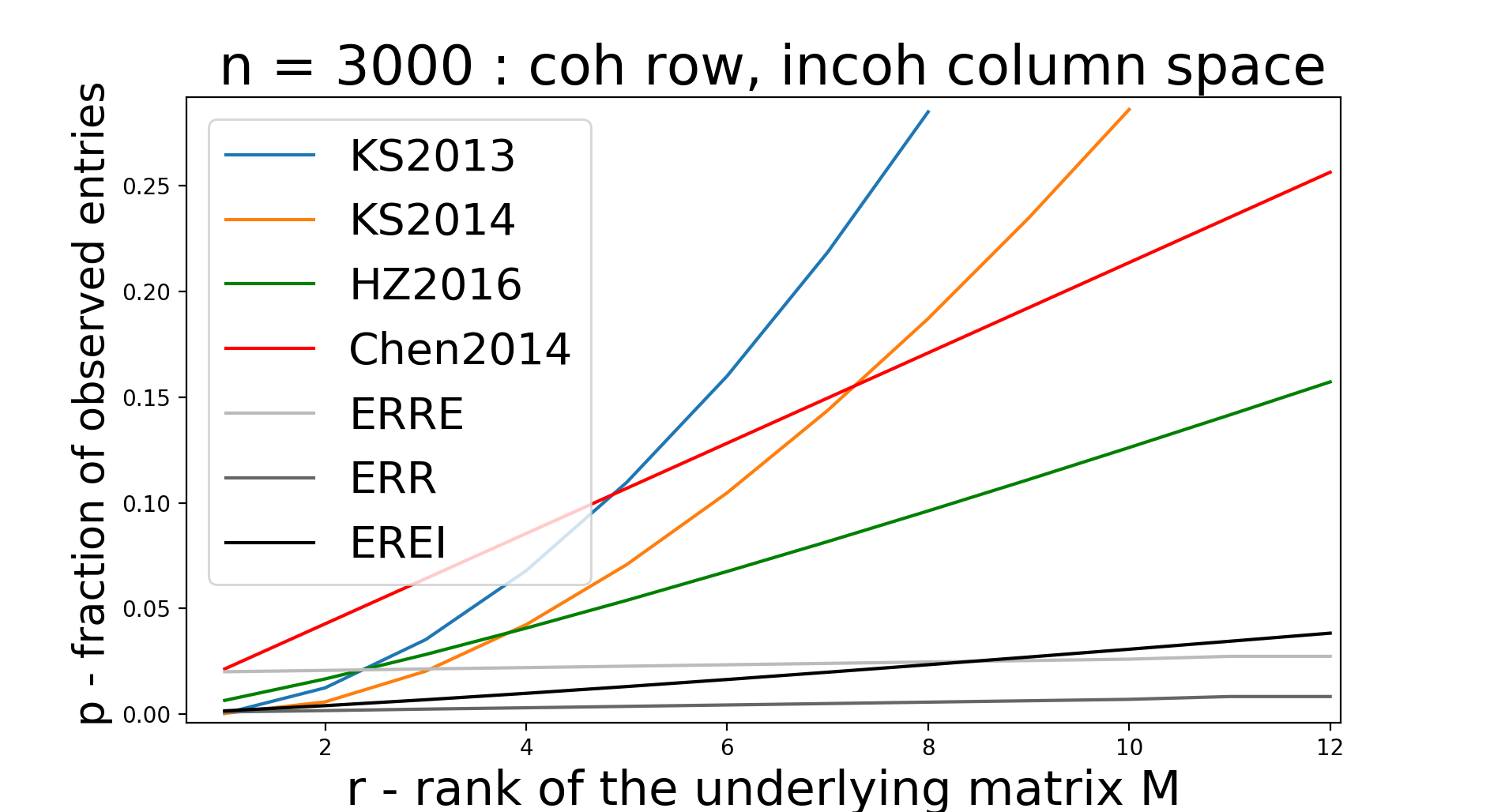}   \\
\includegraphics[height=2.74cm,width=4.2cm]{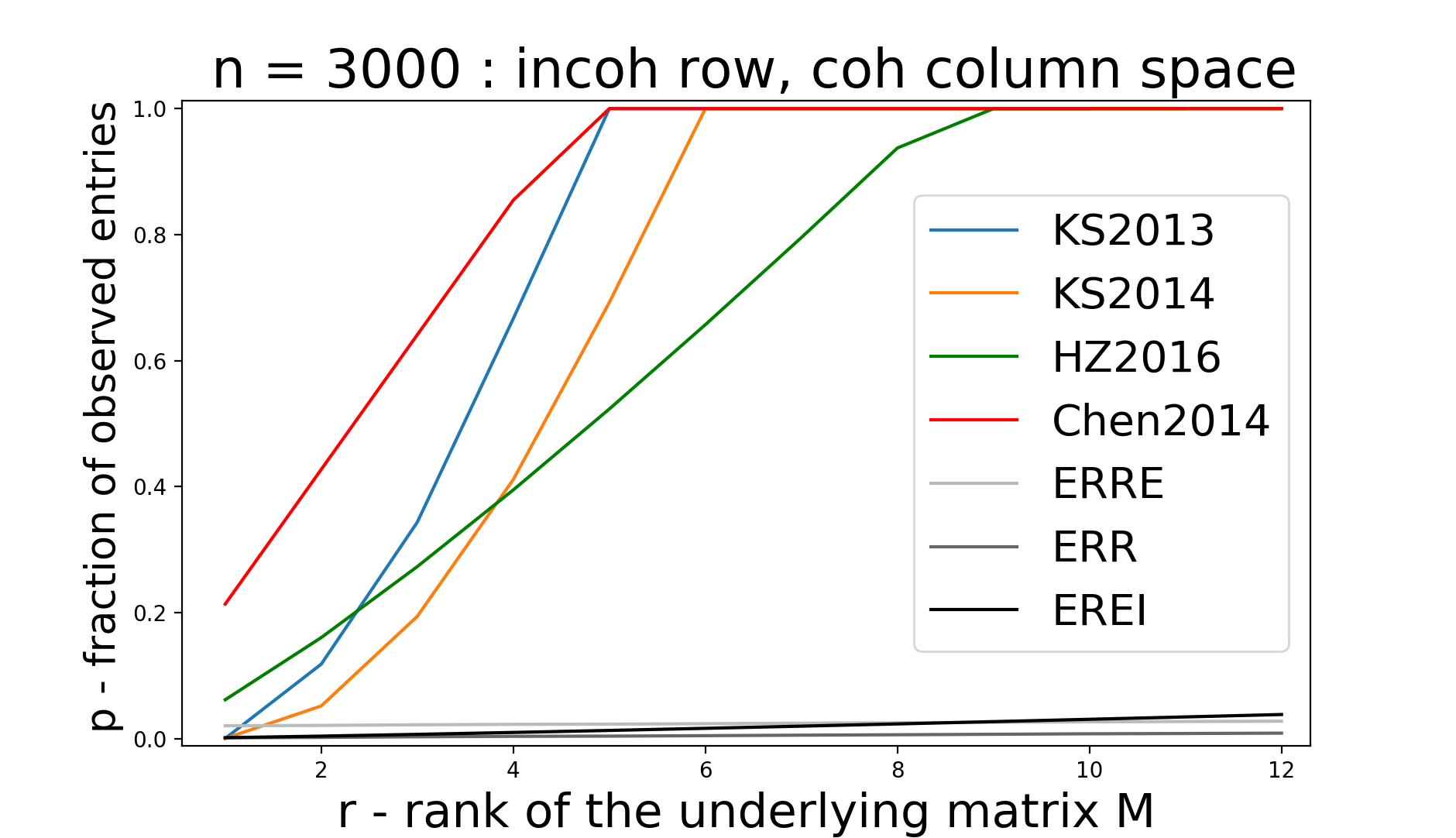}   
\includegraphics[height=2.74cm,width=4.2cm]{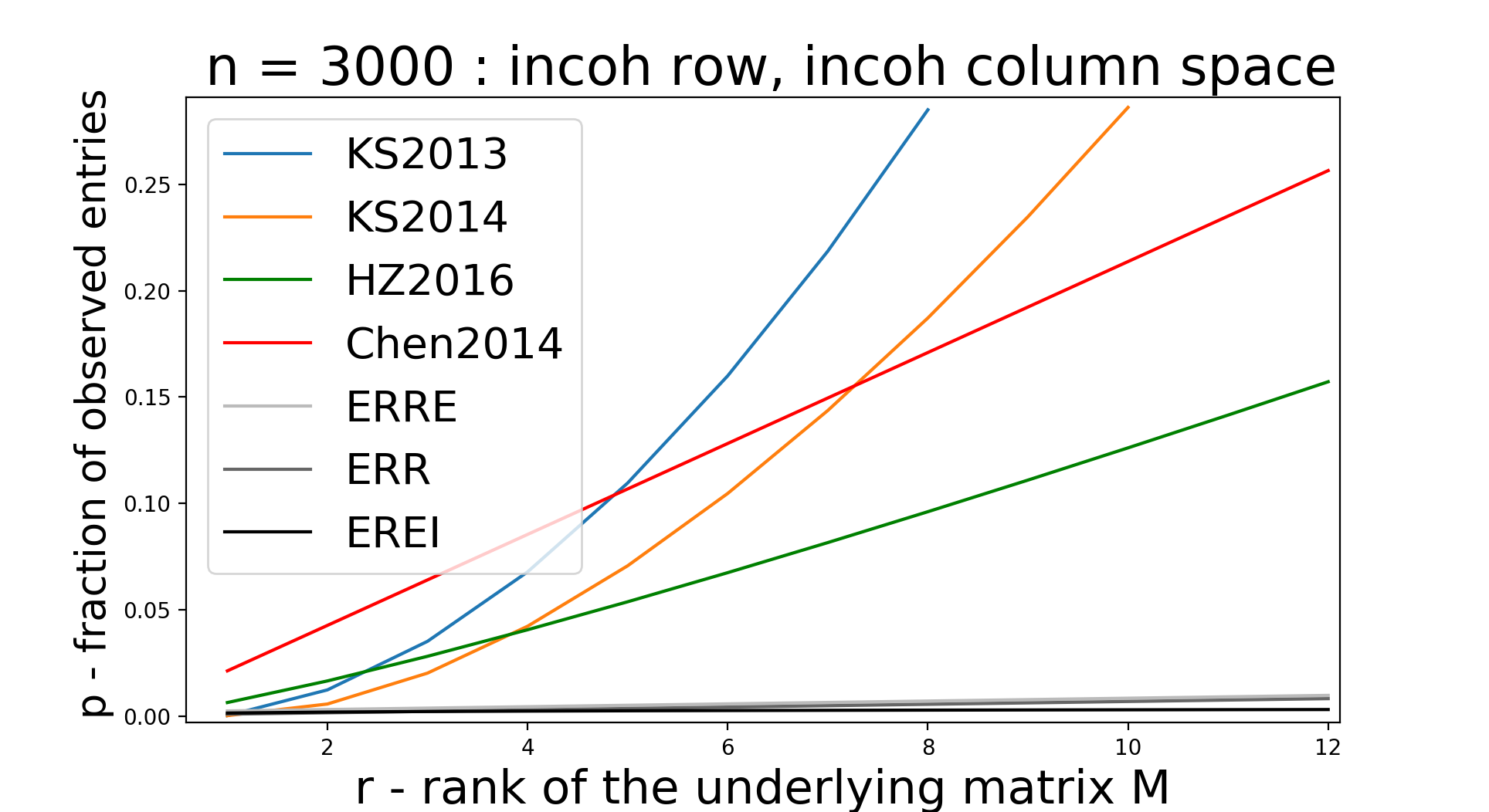} 
\end{tabular}
\end{center}
\end{table}

\section{Conclusion}

In this paper we introduce a new parameter for matrices and subspaces.
Using this parameter, we suggest new adaptive matrix completion algorithms which we are able to do simple combinatorial proofs for observation complexities.
Moreover, we show these bounds are much more efficient than bounds provided in the  state of the art.
In  following works, we are going to extend these algorithms to make them robust for different type of noises.

\bibliography{main}

\appendix
\providecommand{\upGamma}{\Gamma}
\providecommand{\uppi}{\pi}

\section{Properties of sparsity-number}

\subsection{Proof of lemma 1}

\begin{proof}
By the hypothesis of linear dependence, there are coefficients 
$\alpha_1, ... , \alpha_{n}$, not all 
zero, such that 
\begin{align*}
\alpha_1 x^1_\Omega + ... + \alpha_{n} x^n_\Omega = 0.     
\end{align*}
To show linear dependence of $x^1,\ldots, x^n$ we prove the following equation also satisfies
\begin{align*}
\alpha_1 {x^1} + ... + \alpha_{n} {x^n}=0     
\end{align*}
Assume by contradiction 
$y= \sum_{i=1}^{k} \alpha_i {x_i}$
is a nonzero vector.  But, we have
\begin{align*}
y_\Omega= \sum_{i=1}^{k} \alpha_i {x_i}_\Omega =0    
\end{align*}
which implies $\overline{\psi}(\mathbb{U}) \geq |\Omega|$ from the definition of \textit{space sparsity number}. 
However, $|\Omega| > \overline{\psi}(\mathbb{U})$ from the hypothesis of the lemma which concludes a contradiction.
Therefore, the assumption the vector $y$ being nonzero vector cannot be true, then the following satisfies : 
\begin{align*}
 y= \sum_{i=1}^{k} \alpha_i {x_i} = 0   
\end{align*}
\end{proof}
In the following lemma we prove somewhat reverse statement of the lemma 1. 
\begin{lemma}\label{lem:zero-num-lin-dependence} Let $\mathbb{U}$ be a subspace of $\mathbb{R}^{m}$ and $x^1,x^2,...,x^n$ be any set of vectors from $\mathbb{U}$. Then the linear 
dependence of ${x^1}, {x^2},...,{x^n}$ implies linear dependence of 
$x^1_\Omega, x^2_\Omega,...,x^n_\Omega$, for any $\Omega \subset [m]$.
\end{lemma}
\begin{proof}
The proof of the statement is straightforward observation of the fact that 
$$\alpha_1 {x^1} + ... + \alpha_{n} {x^n}=0 $$
implies 
$$\alpha_1 x^1_\Omega + ... + \alpha_{n} x^n_\Omega = 0_\Omega = 0 $$    
\end{proof}
\textbf{Tightness of lemma 1:}
We show example for tightness of the lemma with the following matrix:
\[ \mathbf{M} = \begin{bmatrix}
    1       & 2 & 5 \\
    1       & 2 & 4 \\
    1       & 0 & 4 \\
    1       & 0 & 4
\end{bmatrix} \]
First observation here is columns of $\mathbf{M}$ is linearly independent.
Then, the next observation is that $e_1 = (1,0,0,0)$ is contained in the column space.
Therefore, the \textit{space sparsity number} of the column space of $\mathbf{M}$ is at least equal to $3$.
Using the fact that \textit{space sparsity number} is less then 4 we conclude that column \textit{space sparsity number} is exactly equal to $3$.
Lets check the submatirx $\textbf{M}_{\Omega:}$ where $\Omega=\{2,3,4\}$: 
\[ \mathbf{M_{\Omega:}} = \begin{bmatrix}
    1       & 2 & 4 \\
    1       & 0 & 4 \\
    1       & 0 & 4
\end{bmatrix} \]
Columns of $\mathbf{M_{\Omega:}}$ is linearly dependent (first and third column), however columns of $\mathbf{M}$ is not. 
Then, it follows that there is an example that when $\| \Omega \| =\overline{\psi}(\mathbb{U}) $ but the hypothesis of the lemma 1 not satisfied.
Therefore, the statement of the lemma 1 is tight.

\begin{lemma} \label{lem:bds}
For the column space $\mathbb{U}$ of $m\times n$ sized matrix $\mathbf{M}$ with rank-$r$, the following inequality is satisfied $$r-1 \leq \overline{\psi}(\mathbb{U}) \leq m-1.$$ 
\end{lemma}
\begin{proof}
$\overline{\psi}(\mathbb{U}) < m $ is straightforward because any  nonzero vector in $\mathbb{R}^{m}$ has at most $m-1$ coordinates equal to zero.
Then, it follows from the definition of the \textit{space sparsity number}, $\overline{\psi}(\mathbb{U}) \leq m-1$. \medskip \\
In the rest of the proof we prove $r-1 \leq \overline{\psi}(\mathbf{M})$.
$\mathbf{M}$ having rank $r$ implies that we can choose $r$ rows from it that are the basis for the row space of it. 
Technically, we may find $R\subset [m]$ such that $|R| = r$ and $\mathbf{M}_{R:}$ is rank $r$. 
Similarly, we can find $C\subset [n]$ such that $|C| = r$ and $\mathbf{M}_{R:C}$ is an $r \times r$-sized matrix of rank $r$. 
It follows that there exists $\alpha \in \mathbb{R}^r$ such that $$\mathbf{M}_{R:C} \alpha = e_1 = \big(1,0,\ldots,0\big).$$ 
Consequently, $\mathbf{M}_{:C}\alpha \neq 0$ but has zero components in $r-1$ of the indices given by $R$. 
Thus $\overline{\psi}(\mathbf{M}_{:C} )\geq r-1$. 
Finally: $$ r-1 \leq \overline{\psi}(\mathbf{M}_{:C} ) \leq \overline{\psi}(\mathbb{U})$$
\end{proof}

\section{Proof of Theorem 2}

We split the statement of the theorem above into two and prove each of them separately.
First, we show that observation complexity is upper bounded by $(m+n-r)r +  2 \frac{m n}{{\psi}(\mathbb{U})}\log{(\frac{r}{\epsilon})} $.
Then, in another theorem we show that the observation complexity is bounded by 
$$(m+n-r)r + \frac{\frac{2m}{\psi(\mathbb{U})}(r+2 +\log{\frac{1}{\epsilon}})}{\psi(\mathbb{V})}  $$
and the statement of theorem follows from these two results.

\subsection{Matrices with Incoherent Row Space}

\textit{Theorem:} 
Let $r$ be the rank of underlying $m\times n$ sized matrix $\mathbf{M}$ with column space $\mathbb{U}$. 
Then,  $\mathbf{ERR}$ exactly recovers the underlying matrix $\mathbf{M}$ with probability at least $1-\epsilon$  using at most 
$$(m+n-r)r +  2 \frac{m n}{{\psi}(\mathbb{U})}\log{(\frac{r}{\epsilon})}  $$
observations.

\begin{proof}
The proof is consisting following steps: 
\begin{itemize}
\item \textit{step 1.} Give terminology will be used throughout the proof. Identifying type of observations to two classes : informative and  non-informative.
\item \textit{step 2.} Provide a bound to number of informative observations.
\item \textit{step 3.} In remaining steps, we try to give bound to non-informative observations. We start by giving upper bound to the unsuccessful observations in line 5.
\item \textit{step 4.} We model the execution of  $\mathbf{ERR}$ with a stochastic process and design another process which terminates faster than this.
\item \textit{step 5.} We relate the problem to basic combinatorial counting  problem and analyse
\item \textit{step 6.} Conclude that total number of observations is 
\begin{align*}
(m+n-r)r +  2 \frac{m n}{{\psi}(\mathbb{U})}\log{(\frac{r}{\epsilon})}.
\end{align*}

\end{itemize}
\textit{Step 1:}
For ease of readability we denote $\psi(\mathbb{U})$ by $k$ during the proof.
Lets start a process in the beginning of the algorithm for each column. 
We call process of the column $\mathbf{M}_{:j}$ dies in one of the following cases happens:
\begin{itemize}
    \item[a.] $\mathbf{M}_{:j}$ is fully observed in line 6 in some intermediate step of  $\mathbf{ERR}$
    \item[b.] $\mathbf{M}_{:j}$ is contained in the column space of the already fully observed columns in underlying matrix $\mathbf{M}$ (i.e. columns in $C$).
    \item[c.]  Algorithm already learns entire column space : $\widehat{r} = r $
\end{itemize}
If a column/process is not dead then we call it is active.
We call an observation is informative if it is observed at line 6 (i.e. it contributes to the studied column/row space learned by  $\mathbf{ERR}$ and uninformative if it observed at line 3. 
Obviously some entries are observed both at line 3 and 6, so they count in both non-informative and informative observations. \\\\
\textit{Step 2:}
We can simply observe that the number of informative observations is exactly $m r + n r - r^2$. 
Because, at the end of the algorithm set of informative observations is just set of $r$ many linearly independent columns (we have $mr$ observations here) and $r$ many linearly independent rows (we have $nr$ observations here). 
By observing that entries in $r\times r$ sub-matrix is counted twice, we conclude that overall observations is just 
\begin{align*}
  mr + nr - r^2   
\end{align*}
\textit{Step 3:} In order to give upper bound to the number of non-informative observations, we see it is enough to bound the number of phases the algorithm  $\mathbf{ERR}$ passes through.
Specifically, if the number of phases is bounded by $T$ then overall number of non-informative observations is bounded by $Tn_2$. 
In order to give upper bound to $T$, we first explore the probability of an detecting independence of an observation in line 3 for an active column:
\begin{lemma} \label{lemma:prob_bound}
The probability of detecting independence of an active column in the $j$'th phase of the algorithm  $\mathbf{ERR}$ is lower bounded by $\frac{k}{m-j}$
\end{lemma}
\begin{proof}
In an intermediate step of  $\mathbf{ERR}$ we have $|C| = |R| = \widehat{r}$ and   $M_{R:C}$ is $\widehat{r} \times \widehat{r}$ matrix of rank $\widehat{r}$. 
Then, for any $i\in [n]$, $\mathbf{M}_{R:i}$ is in the column space of $\mathbf{M}_{R:C}$ as the matrix is full rank and therefore its column space is entire $\mathbb{R}^{\widehat{r}}$. 
Then there exists unique coefficients $\alpha_1, \alpha_2, ...,\alpha_{\widehat{r}}$
for columns $C =\{c_1,...,c_{\widehat{r}}\}$ that following equality satisfied.
\begin{align*}&\alpha_1 \mathbf{M}_{R:c_1} + \ldots + \alpha_{\widehat{r}} \mathbf{M}_{R:c_{\widehat{r}}}  = \mathbf{M}_{R:i} \\
 &\implies \hspace{5mm} \alpha_1 \mathbf{M}_{R:c_1} + \ldots + \alpha_{\widehat{r}} \mathbf{M}_{R:c_{\widehat{r}}}  - \mathbf{M}_{R:i} = 0
\end{align*}
Now, let observe the vector 
\begin{align*} 
y=\alpha_1 \mathbf{M}_{:c_1} + \alpha_2 \mathbf{M}_{:c_2}  + ... + \alpha_{\widehat{r}} \mathbf{M}_{:c_{\widehat{r}}}  - \mathbf{M}_{:i} 
\end{align*} 
We know that $y\neq 0$ because we know 
column $i$ is linearly independent with previous observed columns -
$\mathbf{M}_{:c_1}, \mathbf{M}_{:c_2},..., \mathbf{M}_{:c_{\widehat{r}}}$.
Moreover for any row index $a\in R$,  $y_a =0$ because $y_R =0$
from the definition of $\alpha_j$'s. 
As $y$ is in the column space of $M$ it has at most $m-k$-many zero coordinates. 
Moreover, for any $a \notin R$ but $\mathbf{M}_{i_a:i}$ observed in line 6 
$y_{i_a}=0$ should satisfy, because otherwise in one of previous iterations we 
would already decide $\mathbf{M}_{:i}$ is linearly independent and we would add index $i$ to $C$, but here we know $i\notin C$. \medskip\\
Basically we conclude that all known coordinates of $y$ is 0 and number of 
known coordinates is represented by $observed$.
We know at least $k$  many coordinates of $y$ is nonzero and we already have $m-observed$ many coordinates of $y$ is zero, then with probability at least: $\frac{k}{m-observed}$ uniformly selected next observation will be zero.
Being nonzero of $y_a$ implies non-singularity of the matrix $\mathbf{\widehat{M}}_{\widehat{R}:\widehat{C}}$ where 
$\widehat{R}=R\cup\{a\}$ and $\widehat{C}=C\cup \{i\}$. It is because if this matrix was not 
invertible then there would 
be coefficients $\beta_1,...,\beta_{r'+1}$ (not all of them are zero) such that 
\begin{align*}
\beta_1 \mathbf{M}_{\widehat{R}:c_1} + ... + \beta_{\widehat{r}} \mathbf{M}_{\widehat{R}:c_{\widehat{r}}} + \beta_{\widehat{r}+1} \mathbf{M}_{\widehat{R} : i}  = 0.    
\end{align*} 
From lemma 2, linear independence of $\mathbf{M}_{R:c_1}, \mathbf{M}_{R:c_{\widehat{r}}} $ implies linear independence of $\mathbf{M}_{\widehat{R}:c_1}, \mathbf{M}_{\widehat{R}:c_{\widehat{r}}} $.
which concludes $\beta_{\widehat{r}+1}$ is nonzero, so we can simply assume it is $-1$.
Then 
\begin{align*}
 &\beta_1 \mathbf{M}_{\widehat{R}:c_1} + ... + \beta_{\widehat{r}} \mathbf{M}_{\widehat{R}:c_{\widehat{r}}}   =  \mathbf{M}_{\widehat{R} : i} \\
 &\implies  \hspace{5mm}
 \beta_1 \mathbf{M}_{R:c_1} + ... + \beta_{\widehat{r}} \mathbf{M}_{R:c_{\widehat{r}}}   =  \mathbf{M}_{R : i}
\end{align*}
Due to uniqueness of $\alpha_j$'s above, we can tell that 
\begin{align*}
    \alpha_1 = \beta_1 \\
    \alpha_2 = \beta_2 \\ 
    \vdots   \\
    \alpha_{\widehat{r}} = \beta_{\widehat{r}}
\end{align*}
Then the vector
$y_{\widehat{R} }=\alpha_1 \mathbf{M}_{\widehat{R}:c_1} + \ldots + \alpha_{\widehat{r}} \mathbf{M}_{\widehat{R}:c_{\widehat{r}}}   -  \mathbf{M}_{\widehat{R} : i}  = 0$ is a zero vector.
However it is a contradiction because if $y_{\widehat{R}}$ is zero vector  then $y_a=0$ due to $a\in \widehat{R}$ which we already know $y_a\neq 0$.\medskip \\
Therefore, being nonzero of $y_a$ implies non-singularity of $\mathbf{M}_{\widehat{R}:\widehat{C}}$ which is equivalent to the detection of the independence of column $\mathbf{M}_{:i}$ due to lemma 2.
As a conclusion, probability of detection of independence of an active column is at least $\frac{k}{m-observed}$  and considering the fact that $observed \geq j$ it follows that  $\frac{k}{m-observed} > \frac{k}{m-j}$ and it give the final conclusion of the desired probability is lower bounded by:
$\frac{k}{m-j}$. As desired.
\end{proof}
\bigskip
\textit{Step 4 :} We can model execution of $\mathbf{ERR}$ as following stochastic process: 
\begin{align*}S_0= X_{0,1} + X_{0,2} + . . . + X_{0,n}
\end{align*}
where each of the $X_{0,j}$ corresponds to the indicator variable of the activeness of the column $\mathbf{M}_{:j}$.
Obviously, initially at least $r$ of these random variables are equal to 1. 
We define $S_1$ similarly:
\begin{align*}S_1= X_{1,1} + X_{1,2} + . . . + X_{1,n}
\end{align*}
and for any j that $X_{0,j}=1$ satisfied, at this phase $X_{1,j}$ will be equal to 0 with probability at least $\frac{k}{m-0}$ from lemma \ref{lemma:prob_bound}. 
For remaining $j$'s that $X_{0,j}=0$ satisfied then $X_{1,j}=0$ also to be satisfied.
For the next step $S_2$ defined as:
$$S_2= S_1 +  X_{2,1} + X_{2,2} + . . . + X_{2,n_2}$$ 
where again  for any $j$ that $X_{0,j}=1$ satisfied, at this phase $X_{1,j}$ will be equal to $0$ with probability at least $\frac{k}{m-1}$ from lemma \ref{lemma:prob_bound}.
Remaining $j$'s will stay as $X_{2,j}$ to be equal to 0.
In general 
\begin{align*}S_{iter} = X_{iter,1} + X_{iter,2} + . . . + X_{iter,n_2} 
\end{align*}
where again  for any j that $X_{iter-1,j}=1$ satisfied, at this phase $X_{iter,j}$ will be equal to 0 with probability at least $\frac{k}{m-(iter-1)}$ from lemma \ref{lemma:prob_bound}.
The termination of 
algorithm is equivalent to the point $S_p = 0$ in this model.
One can see termination of this process is upper bounded by termination of the 
following process :
\begin{align*}S'_0 = X'_{0,1} + X'_{0,2} + . . . + X'_{0,r}
\end{align*}
where each of the $X'_{0,j}$ is set to be equal to $1$. we define $S'_1$ in a similar way:
\begin{align*}S'_1 = X'_{1,1} + X'_{1,2} + . . . + X'_{1,r}
\end{align*}
where each of $X'_{1,j}$ is equal to $0$ with probability $\frac{k}{m}$. Then:
\begin{align*}S'_2 = X'_{2,1} + X'_{2,2} + . . . + X'_{2,r}.
\end{align*}
Similarly $X'_{2,j}$ is set to be $0$ if $X'_{1,j}=0$ and $X'_{2,j}$ is equal to $0$ with probability $\frac{k}{m-1}$ otherwise.
In general:
\begin{align*}S'_{iter} = X'_{iter,1} + X'_{iter,2} + . . . + X'_{iter,r} 
\end{align*}
again $X'_{iter,j} = 0$ if $X'_{iter-1,j} = 0$ and, $X'_{iter,j}=0$ with probability $\frac{k}{m-(iter-1)}$ otherwise.\\\\\\
\textit{Step 5:} \hypertarget{step5}{Here}  we use a combinatorial argument to bound number of observation in each column.
\begin{lemma}
Let $X'$ be a process that is zero initially: $X'_0 = 1$ and remaining entries defined as 
\begin{align*}
    X'_{i+1} = 
    \begin{cases}
    \begin{cases}
        0 & \text{with probability } \frac{k}{m-i} \\
        1 & \text{otherwise}
    \end{cases}    & \hspace{3mm} \text{if  } \hspace{3mm}  X'_i = 1\\
    0   & \hspace{3mm}  \text{if  } \hspace{3mm}  X'_i = 0
    \end{cases}
\end{align*} 
Then expected point that $X'$ to switch to 0 is $\frac{m+1}{k+1}$.
\end{lemma}

\begin{proof}
Lets denote the expected switch time with $st$ and write the expression for it: 
\begin{align*}
\mathbb{E}[st] = \sum i P(st = i) &= 1 \frac{k}{m}  \\
                                    &+2 \frac{k}{m-1}\Big(1-\frac{k}{m}\Big)  \\ &+3\frac{k}{m-2}\Big(1-\frac{k}{m}\Big)\Big(1-\frac{k}{m-1}\Big)  \\
                                    \vdots
\end{align*}
We claim that this sum is equal to the expected position of the first $1$ in a random binary string with $k$ many $1$ and $m-k$ many $0$.
To observe truth of the claim we notice followings:
\begin{itemize}
    \item First $1$ being in the first position is obviously $\frac{k}{m}$ as there are $k$ many $1$'s out of $m$ many characters.
    \item The probability of the first $1$ being in the second place is $(1-\frac{k}{m})\frac{k}{m-1}$. The first entry being zero has probability: $1-\frac{k}{m}$ and the second entry being one is $\frac{k}{m-1}$
    \item The probability of the first $1$ being in the $i$-th place is 
    $$\Big(1-\frac{k}{m}\Big)\ldots \Big(1-\frac{k}{m-(i-1)}\Big) \Big( \frac{k}{m-(i-1)}\Big).$$ The first entry being zero has probability: $1-\frac{k}{m}$, the second entry being zero has probability $1-\frac{k}{m-1}$ and so on so forth. Finally out of remaining $m-i+1$ entries the next one being $1$ is equal to $\frac{k}{m-(i-1)}$.
\end{itemize}
Then expected position of the first 1 is equal to 
\begin{align*}
 1 \frac{k}{m} + &2 \frac{k}{m-1}\Big(1-\frac{k}{m}\Big) + \\ &3\frac{k}{m-2}\Big(1-\frac{k}{m}\Big)\Big(1-\frac{k}{m-1}\Big) + \ldots   
\end{align*}
which is equal to $\mathbf{E}[st]$.
Lets find the position of the first $1$ by double counting technique. 
A word with $k$ number of $1$ and $m - k$ number of $0$ 
can be represented as $a_0 1 a_1 1 a_2 . . . 1 a_k$ where $a_i$ represents number 
of zeros between two $1$'s. 
Now, lets find number of first 1 in the $k+1$ sized set of following words
\begin{align*}
& a_0 1 a_1 1 a_2 . . . 1 a_k,  \\
& a_1 1 a_2,...,a_k 1 a_0, \\
& a_2 1 a_3,...,a_0 1 a_1, \\
& \vdots \\
& a_k 1 a_0,...,a_k 1 a_0 
\end{align*}
Expected number of first 1 here is simply \medskip 
\begin{align*}
\frac{a_0+1}{k+1} +  \ldots +\frac{a_k+1}{k+1} 
&=  \frac{a_0+a_1+ ... +a_k + k+1}{k+1}  \\ 
&= \frac{m+1}{k+1}.
\end{align*} \medskip
So, we can divide set of all words with $k$ many $1$'s and $m-k$ many $0$'s into $k+1$-sized sets.
For each group the average position of the first 1 will be $\frac{m+1}{k+1}$.
Therefore, in overall the average position of the first $1$ is $\frac{m+1}{k+1}$.
\end{proof}
A simple followup of this lemma is  to notice:
$$\mathbb{E}[st]=\frac{m+1}{k+1}< \frac{m}{k}$$  due to $m>k$. 
Then, we can use the Markov inequality to get: $$P\Big(st>2\frac{m}{k}\Big) < \frac{1}{2}.$$
Moreover, from the combinatorial counting argument we can imply that the probability of $st > a$ will be given as 
\begin{align*}
P(st > a) = \frac{ \binom{m-a}{k} } { \binom{m}{k} }
\end{align*}
using the previous inequality we can observe that:
\begin{align*}
P\Big(st > \frac{2m}{k}\Big) = \frac{ \binom{m-2m/k }{k} } { \binom{m}{k} } < \frac{1}{2}
\end{align*}
Considering the fact 
\begin{align*}
f(x) = \frac{ \binom{x-2m/k}{k} } { \binom{x}{k} }
\end{align*}
is an increasing function and
\begin{align*}
P\Big(st > \frac{\alpha m}{k}\Big) &= 
\frac{ \binom{m-2m/k }{k} } { \binom{m}{k} } 
\frac{ \binom{m-4m/k }{k} } { \binom{m-2m/k}{k} }
\cdots
\frac{ \binom{m-2 \alpha m/k }{k} } { \binom{m-2(\alpha-1)m/k}{k} } \\
&< (\frac{1}{2})^\alpha 
\end{align*}
For a given $\epsilon$, if we set $\alpha = \log{\frac{1}{\epsilon}}$ we conclude that with probability at least $1-\epsilon$ the following inequality satisfied:
\begin{align*}
st > 2 \frac{m}{k} \log{\frac{1}{\epsilon}}.    
\end{align*}
\\
\textit{Step 6:} So, we can tell 
\begin{align*}
P\Big(X' \geq 2 \log{(\frac{1}{\epsilon})} \frac{m}{k}\Big) \leq \epsilon.    
\end{align*}
Which means for a given $j$, with probability more than $1-\epsilon$,  $X'_{i,j}$ will switch to zero before   $ 2 \log(\frac{1}{\epsilon})\frac{m}{k}$ for any $ j \in [r]$. 
Using union bound argument, after $2\log(\frac{1}{\epsilon})\frac{m}{k}$ 
iteration with probability more than $1-\epsilon r$, for any $i\in [r]$, $X'_{i,j}$ will switch to zero.\\\\
Therefore, the process $S$ will stop before $2\log{\frac{r}{\epsilon}}\frac{m}{k}$ iteration with probability $1-\epsilon$. 
Remind that, termination time of $S$ corresponds to the value of $T$ and number of total red points is bounded by $Tn$.
Then number of total red observations is bounded by: $$2 \frac{m n}{k} \log{\frac{r}{\epsilon}}.$$
Finaly, total number of observations is equal to the number of red observations plus number of blue observations which gives the bound:
\begin{align*}
(m+n-r)r+  2 \frac{m n}{k} \log{\frac{r}{\epsilon}}. 
\end{align*}
\end{proof}

\subsubsection{Discussion for Incoherent Row Spaces:}
In the lemma \ref{cohbd} we show that $$\mu(\mathbb{U}) \geq \frac{m}{r} \frac{1}{\psi(\mathbb{U})}.$$ and in the theorem above we prove that the observation complexity of $\mathbf{ERR}$ is upper bounded by $(m+n-r)r + 2 \frac{m n}{\psi(\mathbb{U})} \log{\frac{r}{\epsilon}}$ where $\mathbb{U}$ is column space of the matrix $\mathbf{M}$. \\\\
Lets denote the fraction $$\gamma = \frac{m}{\psi(\mathbb{U})} \frac{1}{\mu(\mathbb{U}) r}$$ then lemma \ref{cohbd} is equivalent to $\gamma \leq 1 $.
Lets transfer observation complexity of $\mathbf{ERR}$ with respect to $\mu(\mathbb{U})$ using $\gamma$.
Then we get the observation complexity is $$(m+n-r)r + 2\gamma \mu(\mathbb{U}) r \log{r/\epsilon}$$ and using the fact that $\gamma \leq 1$ this number is smaller than bound due to \cite{nina}:
$$(m+n-r) + 2 \mu(\mathbb{U}) r \log{r/\epsilon}$$
In many cases $\gamma$ can be very small. 
For any matrix that has high value of-$\psi(\mathbb{U})$ or low value of $\mu(\mathbb{U})$,  $\gamma$ is guaranteed to be very small. 
Specifically, if $\psi(\mathbb{U})$ is $\Theta(m)$ or $\mu(\mathbb{U})$ is $\Theta{\frac{m}{r}}$ then $\gamma$ is $\mathcal{O}(\frac{1}{r})$.
Proofs for each case provided below:\\\\
\textit{$\psi(\mathbb{U})$ is $\Theta(m)$}:  Assigning $\psi(\mathbb{U})$ being $\Theta(m)$ in the definition of $\gamma$, we conclude that $\gamma$ is $\Theta(\frac{1}{\mu(U)r})$.
Remember from the definition of the coherence, $\mu(\mathbb{U}) \geq 1$ for any subspace, which gives the final conclusion of $\gamma$ is $\mathcal{O}(\frac{1}{r})$. \medskip\\
\textit{$\mu(\mathbb{U})$ is $\Theta{\frac{m}{r}}$}:  Assigning $\mu(\mathbb{U})$ being $\Theta(\frac{m}{r})$ in the definition of $\gamma$, we conclude that $\gamma$ is $\Theta(\frac{1}{\psi(\mathbb{U})})$.
Moreover, remember that from lemma \ref{lem:bds}, we know that $\psi(\mathbb{U})$ is $\Omega(r)$ which gives final conclusion of $\mathcal{O}(\frac{1}{r})$. \medskip\\

\subsection{Matrices with Coherent Row Space}

We use the same terminology as previous theorem and $k$ and $t$  stands for $\psi(\mathbb{U})$ and $\psi(\mathbb{V})$ correspondingly. 
So, if a column is not in the column space of $C$ then we call it active.\medskip\\
From lemma \ref{lemma:prob_bound}, we know that at any step if a column is still active, then probability of its detection is at least $\frac{k}{m}$ where $k$ is the \textit{space non-sparsity number} for column space.
Let's just focus on active observations, and estimate the number of required active observations to detect $r$-th linearly independent column.
We can see that, under the condition of each observation being active observation and the probability of detection being exactly $\frac{k}{m}$ the process of the detection of $r$-th independent column can be modelled as negative binomial distribution.\medskip\\
Lets remind the formula of the probability mass function negative binomial distribution as getting $a$-th success in the $a+b$'th step while success probability being $p$ :
$$ f(a,b,p) = \binom{a+b-1}{a}p^a (1-p)^b$$
For this problem, we are interested to find the probability for finding $r$-th success at $N$-th trial
which corresponds to :
$$ f(r,N-r,\frac{k}{m}) = \binom{N-1}{r}{ \Big( \frac{k}{m}\Big) }^r \Big( 1- \frac{k}{m}\Big) ^{N-r-1}$$
As the number of observations is the focus of this theorem, we fix parameters $k,m,r$ and investigate the behaviour of the function while $N$ being variable.
Intuitively, we use the following notation: 
$$\tau_{k,m,r}(N) =  f(r,N-r,\frac{k}{m})$$
In lemma \ref{lm:taubd} and \ref{1/n} we investigate properties of this function to have better understanding of failure probability of $\mathbf{ERR}$: \\
\begin{lemma} \label{lm:taubd}
$\tau_{k,m,r}(N)$ is a decreasing function after N being larger than $(\frac{2m}{k}+1) r$. Specifically,
we can give the following bound for the decreasing rate: 
$$ 1-\frac{k}{m} < \frac{\tau_{k,m,r}(N+1)}{\tau_{k,m,r}(N)} < 1-\frac{k}{2m} $$
\end{lemma}
\begin{proof}
To show the decreasing we analyse the fraction : \\
\begin{align*}
  \frac{\tau_{k,m,r}(N+1)}{\tau_{k,m,r}(N)}  &= \frac{ \binom{N}{r}{ \Big( \frac{k}{m}\Big) }^r \Big( 1- \frac{k}{m}\Big) ^{N-r} }
{\binom{N-1}{r}{ \Big( \frac{k}{m}\Big) }^r \Big( 1- \frac{k}{m}\Big) ^{N-r-1} } \\
&=\frac{ \frac{N!}{r!(N-r)!}{ \Big( \frac{k}{m}\Big) }^r \Big( 1- \frac{k}{m}\Big) ^{N-r} }{\frac{(N-1)!}{r!(N-1-r)!}{ \Big( \frac{k}{m}\Big) }^r \Big( 1- \frac{k}{m}\Big) ^{N-r-1}}\\[0.7ex]
&=\frac{N}{N-r} (1-\frac{k}{m})
\end{align*}
So, we get following recursive formula:
\begin{align*}
\tau_{k,m,r}(N+1)  =\frac{N}{N-r} \Big(1-\frac{k}{m}\Big){\tau_{k,m,r}(N)}.    
\end{align*}
Left side of the the target inequality is easy to prove as $\frac{N}{N-r}>1$ implies
$$ \frac{\tau_{k,m,r}(N+1)}{\tau_{k,m,r}(N)}  > 1-\frac{k}{m}. $$
Then, we only need to prove the right side of the inequality.
Lets make the following observations  $$\frac{N}{N-r} = 1 + \frac{r}{N-r}$$
and from the hypothesis of the lemma we have $$N > \Big(\frac{2m}{k}+1\Big)r  \implies N-r > \frac{2m}{k}r \implies \frac{r}{N-r} < \frac{k}{2m}. $$
Now, we are ready to prove rigth side:\\
\begin{align*}
\frac{\tau_{k,m,r}(N+1)}{\tau_{k,m,r}(N)}  &= \frac{N}{N-r}\Big(1-\frac{k}{m}\Big) \\
&= \Big(1+\frac{r}{N-r}\Big)\Big(1-\frac{k}{m}\Big)  \\
&< \Big(1+\frac{k}{2m}\Big)\Big(1-\frac{k}{m}\Big)   \\
& = 1-\frac{k}{2m}-\frac{k^2}{2m^2}  \\
& < 1-\frac{k}{2m}.
\end{align*}
 Therefore:\\
 \begin{align*}
  1-\frac{k}{m} < \frac{\tau_{k,m,r}(N+1)}{\tau_{k,m,r}(N)} < 1-\frac{k}{2m}.    
 \end{align*}
 Note that we can claim decreasing of $\tau_{k,m,r}$ just follows from the right side of the inequality.
\end{proof}
To explore more properties of the function $\tau_{k,m,r}$ we prove the following lemma.

\begin{lemma} \label{1/n}
Lets assume $n$ is a positive integer. Then $\tau_{k,m,r}$ satisfies the following inequality 
\begin{align*}
 \tau_{k,m,r}\Big(\frac{2m}{k}(r+1)+n\Big) \leq \frac{1}{n}.   
\end{align*}
\end{lemma}
\begin{proof}
It is clear that  $\tau_{k,m,r}\Big(\frac{2m}{k}(r+1)\Big) < 1$ as it is value of a probability mass function.
In lemma 3 we proved that the functions $\tau_{k,m,r}$ is decreasing after $\frac{2m}{k}(r+1)$.
Therefore, for any positive integer $n$ the following inequalities satisfied:
\begin{align*}
\tau_{k,m,r}\Big(\frac{2m}{k}(r+1)+n\Big) &< \tau_{k,m,r}\Big(\frac{2m}{k}(r+1)+n-1\Big) \\
\tau_{k,m,r}\Big(\frac{2m}{k}(r+1)+n\Big) &< \tau_{k,m,r}\Big(\frac{2m}{k}(r+1)+n-2\Big) \\
                               &\vdots \\
\tau_{k,m,r}\Big(\frac{2m}{k}(r+1)+n\Big) &< \tau_{k,m,r}\Big(\frac{2m}{k}(r+1)\Big) 
\end{align*}
By summing all these inequalities we conclude:
\begin{align*}
   n \tau_{k,m,r}\Big(\frac{2m}{k}(r+1)+n\Big) < \sum_{i=1}^n \tau_{k,m,r}\Big(\frac{2m}{k}(r+1)+i\Big) 
\end{align*}
To bound the second term, we can use: $$\sum_{i=1}^n \tau_{k,m,r}\Big(\frac{2m}{k}(r+1)+i\Big) \leq  \sum_{i=r}^{\infty}\tau_{k,m,r}(i) = 1 $$
and dividing left and rigth side of the inequality above concludes:
\begin{align*}
    n \tau_{k,m,r}\Big(\frac{2m}{k}(r+1)+n\Big) < 1.
\end{align*}
\end{proof}
To apply the lemma above for $n = \frac{2m}{k}$ we get, $$\tau_{k,m,r}\Big(\frac{2m}{k}(r+1)+\frac{2m}{k}\Big) < \frac{k}{2m}$$
Using right side of the lemma \ref{lm:taubd} we notice :
$$\tau_{k,m,r}\Big(\frac{2m}{k}(r+1)+\frac{2m}{k} + i\Big)< \frac{k}{2m}\Big(1-\frac{k}{2m}\Big)^i$$
for any positive integer $i$. Picking $i = \frac{2m}{k} \log{\frac{1}{\epsilon}}$ follows as :
\begin{align*}
 \tau_{k,m,r}\Big(\frac{2m}{k}(r+1)+\frac{2m}{k} &+  \frac{2m}{k} \log{\frac{1}{\epsilon}} \Big)\\  &< \frac{k}{2m}  \Big(1-\frac{k}{2m}\Big)^{\frac{2m}{k} \log{\frac{1}{\epsilon}}}\\[0.5ex]   
&<  \frac{k}{2m}  e^{- \log{\frac{1}{\epsilon}}}  
=  \frac{k}{2m}  \epsilon  
\end{align*}
second inequality here is application of the $(1- \frac{1}{\alpha})^{\alpha} < \frac{1}{e}$ for $\alpha > 0$. Therefore we currently have :
$$ \tau_{k,m,r}\Big(\frac{2m}{k}(r+1)+\frac{2m}{k} +  \frac{2m}{k} \log{\frac{1}{\epsilon}}\Big)  <  \frac{k}{2m}  \epsilon  $$
and we target to bound :
$$ \sum_{i=0}^{\infty} \tau_{k,m,r}\Big(\frac{2m}{k}(r+1)+\frac{2m}{k} +  \frac{2m}{k} \log{\frac{1}{\epsilon}} + i \Big) .  $$
To apply right side of lemma \ref{lm:taubd}, $i$ times we conclude :\\
\begin{align*}
 \tau_{k,m,r}\Big(&\frac{2m}{k}(r+1)+\frac{2m}{k} +  \frac{2m}{k} \log{\frac{1}{\epsilon}} + i \Big) <\\ &\tau_{k,m,r}\Big(\frac{2m}{k}(r+1)+\frac{2m}{k} +  \frac{2m}{k} \log{\frac{1}{\epsilon}} \Big)\Big(1-\frac{k}{2m}\Big)^i
\end{align*}\\
Therefore the summation above can be upper bounded as: 
\begin{align*}
 &\sum_{i=0}^{\infty} \tau_{k,m,r}\Big(\frac{2m}{k}(r+1)+\frac{2m}{k} +  \frac{2m}{k} \log{\frac{1}{\epsilon}} + i \Big)    <\\ 
 &<\sum_{i=0}^{\infty}  \tau_{k,m,r}\Big(\frac{2m}{k}(r+1)+\frac{2m}{k} +  \frac{2m}{k} \log{\frac{1}{\epsilon}} \Big)\Big(1-\frac{k}{2m}\Big)^i  \\
&= \tau_{k,m,r}\Big(\frac{2m}{k}(r+1)+\frac{2m}{k} +  \frac{2m}{k} \log{\frac{1}{\epsilon}} \Big) \sum_{i=0}^{\infty} \Big(1-\frac{k}{2m}\Big)^i  \\
&= \tau_{k,m,r}\Big(\frac{2m}{k}(r+1)+\frac{2m}{k} +  \frac{2m}{k} \log{\frac{1}{\epsilon}} \Big) \frac{2m}{k} \\
&< \epsilon \frac{k}{2m}  \frac{2m}{k}  = \epsilon.
\end{align*}
Therefore, we conclude that the probability of $\mathbf{ERR}$ terminating after  $\frac{2m}{k}\big(r+2 +\log{\frac{1}{\epsilon}}\big)$ is smaller than $\epsilon$.\\\\
At this point we have number upper bound for number of active observations in order to have $1-\epsilon$ probability of termination.
However, we need to give the bound with respect to number of overall observations.
Following lemma will help us for that purpose 

\begin{lemma} \label{lem:activecolumns}
At every phase of the algorith-$\mathbf{ERR}$, if there is at least one active observation, then there is at least $t$ many active observations.
\end{lemma}
\begin{proof} 
The first step is to observe that, any column in $C$ is already inactive as they are already in temporary column space.
The second observation is any column that is linear combination of columns in C also already inactive.
We prove the lemma by assuming the hypothesis of the lemma is not correct and we will deduce contradiction from that.
Therefore, we assume that there is a step that the number of active columns is less than $t$, under the condition not all columns are inactive.\\\\
Number of active columns  being smaller than $t$ implies that the number of inactive columns is larger than $n-t$.
Which implies there is a subset of columns- $\Omega'$ that satisfies $|\Omega'|>n-t$ and $\mathbf{M}_{:\Omega'}$ has rank of at most $r-1$ (as there are still some active columns).\\\\
We know that the rank of $\mathbf{M}$ being $r$ implies there is at least one set of $r$ many linearly independent rows.
Lets denote one of these sets by $R=\{j_1,j_2,...,j_r\}$ and naturally, the set of row vectors
$\mathbf{M}_{j_1:}, \mathbf{M}_{j_2:},..., \mathbf{M}_{j_r:}$ are linearly dependent.\\\\
Returning back to the argument $M_{:\Omega'}$ having a rank of at most $r-1$, implies the rank of $M_{R:\Omega'}$ is also at most $r-1$.
Therefore, there is a linear dependence relation among the vectors $\mathbf{M}_{j_1:\Omega'}$, $\mathbf{M}_{j_2:\Omega'}$, . . . ,$\mathbf{M}_{j_r:\Omega'}$. 
As we already have $\Omega'>n-t$ then using lemma 1 we conclude that there is linear dependence relation among $\mathbf{M}_{j_1:}, \mathbf{M}_{j_2:},..., \mathbf{M}_{j_r:}$ which is a contradiction.
Therefore, if there is one active column we can conclude there is at least $t$ many active columns.
\end{proof}
Rest of the proof is simple counting argument.
We know that if we have $\frac{2m}{k}\big(r+2 +\log{\frac{1}{\epsilon}}\big)$ many observations then with probability larger than $1-\epsilon$ our algorithm succeeds.
Moreover, from the lemma above, as at each phase we have at least $\psi(\mathbb{V})$ many observations, $$\frac{\frac{2m}{k}\big(r+2 +\log{\frac{1}{\epsilon}}\big)}{t}$$ many phase is enough to have desired number of active observations.
Note that, at each step we have at most n many observation, which concludes the statement
$$\frac{\frac{2m}{k}\big(r+2 +\log{\frac{1}{\epsilon}}\big)}{t}n$$ many observation is enough to guarantee with probability $1-\epsilon$

To translate this result to coherence number rather than \textit{space sparsity number}, we use the following lemma:
\begin{lemma}  \label{cohbd}
Let $\mathbb{U}$ be an $r$-dimensional subspace of $\mathbb{R}^{m}$. Then the below relation between $\psi(\mathbb{U})$ and $\mu(U)$  holds:
\begin{align*}
\mu(U) \geq \frac{m}{r} \frac{1}{\psi(\mathbb{U})}.
\end{align*}
\end{lemma}

\begin{proof}
We again denote $\psi(\mathbb{U})$ with $k$ for ease of reading.
By the definition of the 
\textit{space sparsity number}, we see that there exists a vector $v\in U$ and $k$ different indices $i_1,i_2,...,i_k$ 
such that the only nonzero components of $v$ are $v_{i_1},v_{i_2},...,v_{i_k}$. 
Up to scaling, we may assume that $v$ is 
a unit vector. 
This is equivalent to  
\begin{align*}
 {v_{i_1}}^2 + ... + {v_{i_k}}^2 = 1   
\end{align*}
Therefore, we observe that there is an index $i_a$ satisfies ${v_{i_a}}^2 \geq \frac{1}{k}$. 
If this was not the case, then for all $j$ with $1\leq j \leq k$,  ${v_{i_j}}^2 < 
\frac{1}{k}$ should satisfy, and this implies 
$$1={v_{i_1}}^2 + ... + {v_{i_k}}^2 < k \frac{1}{k} = 1$$
and this is a contradiction. 
Using these facts, we can see that 
\begin{align*}
 || P_U e_{i_a} ||^2 \geq || v \cdot e_{i_a} ||^2 = | v \cdot e_{i_a} |^2 = {v_{i_a}}^2 \geq \frac{1}{k}   
\end{align*}
where $e_{i_a}$ is $i_a$'th standard basis of $\mathbb{R}^{m}$. 
The first inequality follows from the fact that the length of projection of any vector to the subspace 
$\mathbb{U}$ is always greater or equal than the length of the projection onto a vector of that subspace. 
Thus, we have 
\begin{align*}
\mu(U) &= \frac{m}{r} \underset{1 \leq j \leq m}{\max}  || P_U e_{j}||^2  \\
 &\geq \frac{m}{r}|| P_U e_{i_a}||^2 \geq \frac{m}{r} \frac{1}{k}
= \frac{m}{r} \frac{1}{\psi(\mathbb{U})}    
\end{align*}
\end{proof}

\subsubsection{Discussion for Incoherent Row Spaces:}
In the lemma \ref{cohbd} we show that $$\mu(\mathbb{U}) \geq \frac{m}{r} \frac{1}{\psi(\mathbb{U})}.$$ and in the theorem above we prove that the observation complexity of $\mathbf{ERR}$ is upper bounded by $(m+n-r)r + 2 \frac{m n}{\psi(\mathbb{U})} \log{\frac{r}{\epsilon}}$ where $\mathbb{U}$ is column space of the matrix $\mathbf{M}$. \\\\
Lets denote the fraction $$\gamma = \frac{m}{\psi(\mathbb{U})} \frac{1}{\mu(\mathbb{U}) r}$$ then lemma \ref{cohbd} is equivalent to $\gamma \leq 1 $.
Lets transfer observation complexity of $\mathbf{ERR}$ with respect to $\mu(\mathbb{U})$ using $\gamma$.
Then we get the observation complexity is $$(m+n-r)r + 2\gamma \mu(\mathbb{U}) r \log{r/\epsilon}$$ and using the fact that $\gamma \leq 1$ this number is smaller than bound due to \cite{nina}:
$$(m+n-r) + 2 \mu(\mathbb{U}) r \log{r/\epsilon}$$
In many cases $\gamma$ can be very small. 
For any matrix that has high value of-$\psi(\mathbb{U})$ or low value of $\mu(\mathbb{U})$,  $\gamma$ is guaranteed to be very small. 
Specifically, if $\psi(\mathbb{U})$ is $\Theta(m)$ or $\mu(\mathbb{U})$ is $\Theta{\frac{m}{r}}$ then $\gamma$ is $\mathcal{O}(\frac{1}{r})$.
Proofs for each case provided below:\\\\
\textit{$\psi(\mathbb{U})$ is $\Theta(m)$}:  Assigning $\psi(\mathbb{U})$ being $\Theta(m)$ in the definition of $\gamma$, we conclude that $\gamma$ is $\Theta(\frac{1}{\mu(U)r})$.
Remember from the definition of the coherence, $\mu(\mathbb{U}) \geq 1$ for any subspace, which gives the final conclusion of $\gamma$ is $\mathcal{O}(\frac{1}{r})$. \medskip\\
\textit{$\mu(\mathbb{U})$ is $\Theta{\frac{m}{r}}$}:  Assigning $\mu(\mathbb{U})$ being $\Theta(\frac{m}{r})$ in the definition of $\gamma$, we conclude that $\gamma$ is $\Theta(\frac{1}{\psi(\mathbb{U})})$.
Moreover, remember that from lemma \ref{lem:bds}, we know that $\psi(\mathbb{U})$ is $\Omega(r)$ which gives final conclusion of $\mathcal{O}(\frac{1}{r})$. \medskip\\

\subsection{Proof of corollary \ref{cor:err} }

We start by rephrasing the corollary 3 and prove the statement for each case later on.\\[1.5ex]
\textit{Corollary:} 
Observation complexity of $\mathbf{ERR}$ studied for three different case below:
\begin{itemize}
    \item if  $ \psi(\mathbb{V})=\mathcal{O}(1)$ then observation complexity is  upper bounded by $(m+n-r)r+2 \frac{m n}{\psi(\mathbb{U})}\log{(\frac{r}{\epsilon})} =(m+n-r)r + \mathcal{O}\big(nr\mu_0\log{(\frac{r}{\epsilon})}\big)$.
    ( this bound matches with \cite{nina} but many times it is much smaller quantity as we discussed in \textbf{C.2})

    \item  if  $ \psi(\mathbb{V})=\Theta(r)$ then observation count is upper bounded by $(m+n-r)r+   \mathcal{O}\Big(\frac{\frac{2m}{\psi(\mathbb{U})}(r+2 +\log{\frac{1}{\epsilon}})}{r}n\Big) =
(m+n-r)r + \mathcal{O}\big(n \mu_0(r+\log{\frac{1}{\epsilon}}) \big) $. Selecting $\epsilon = \frac{1}{2^{\mathcal{O}(r)}}$ gives bound of  $mr + \mathcal{O}(n \mu_0r)$

    \item if  $ \psi(\mathbb{V})=\Theta(n)$ then observation count is upper bounded by  $(m+n-r)r+  \mathcal{O}\Big(\frac{\frac{2m}{\psi(\mathbb{U})}(r+2 +\log{\frac{1}{\epsilon}})}{n}n\Big)  =
(m+n-r)r +  \mathcal{O}\big( \mu_0 r(r+\log{\frac{1}{\epsilon}}) \big) $ Selecting $\epsilon = \frac{1}{2^{\mathcal{O}(r)}}$ gives bound: $\mathcal{O}((m+n-r)r)$


\end{itemize}

\begin{proof}
\textit{Case} : $ \psi(\mathbb{V})=\mathcal{O}(1)$.
From the theorem \ref{thm:lg2} the observation complexity is upper bounded by $$(m+n-r)r +  \mathrm{min}\Big(  2 \frac{m n}{{\psi}(\mathbb{U})}\log{(\frac{r}{\epsilon})} , \frac{\frac{2m}{\psi(\mathbb{U})}(r+2 +\log{\frac{1}{\epsilon}})}{\psi(\mathbb{V})}n \Big)  $$
therefore it is upper bounded by $(m+n-r)r+2 \frac{m n}{\psi(\mathbb{U})}\log{(\frac{r}{\epsilon})}$.
Moreoover, in lemma \ref{cohbd} we show $$\mu(\mathbb{U}) \geq \frac{m}{r} \frac{1}{\psi(\mathbb{U})}$$ which upper bounds the last quantity by 
$(m+n-r)r + \mathcal{O}\big(nr\mu_0\log{(\frac{r}{\epsilon})}\big)$. \\\\
\textit{Case}  $ \psi(\mathbb{V})=\Theta(r)$: This time we choose the second term in $\mathrm{min}$ operator of theorem \ref{thm:lg2}. 
We note that $\frac{m}{\psi(\mathbb{U})}$ can be upper bounded by $\mu(\mathbb{U})r$ and plugging it together with $\psi(\mathbb{V})=\Theta(r)$ gives us upper bound of $(m+n-r)r + \mathcal{O}\big(n \mu_0(r+\log{\frac{1}{\epsilon}}) \big) $. 
Moreover, if  $\epsilon = \frac{1}{2^{\mathcal{O}(r)}}$ then $\log{\frac{1}{\epsilon}}$ is $\mathcal{O}(r)$, therefore right summand is bounded by  $\mathcal{O}(n \mu_0r)$. Considering the fact $\mu_0 \geq 1$ always, then overall expression is upper bounded by $mr + \mathcal{O}(n \mu_0r)$
\\\\
\textit{Case}  $ \psi(\mathbb{V})=\Theta(n)$: This case is just similar too previous case with the difference of plugging  $\psi(\mathbb{V})=\Theta(n)$ gives us the bound of $(m+n-r)r +  \mathcal{O}\big( \mu_0 r(r+\log{\frac{1}{\epsilon}}) \big)$. 
Using the similar bound to $\epsilon$ makes the right summand to be  $\mathcal{O}(\mu_0 r^2)$.
Moreover from the definition of coherence we have $\mu_0 \leq \frac{m}{r}$ which upper bounds this term by $\mathcal{O}(m r)$ therefore overall sum is upper bounded by $\mathcal{O}((m+n-r)r)$.
\end{proof}

\section{Proof of Theorem \ref{thm:erre}. Exact Recovery While Rank Estimation}

\begin{proof}
We again use $k$ and $t$ for $\psi(\mathbb{U})$ and $\psi(\mathbb{V})$ correspondingly and use all the terminology from the previous proofs. 
Then, we start by proving that under the scenario there is still active column, then with probability $1-e^{-T \frac{kt}{m}}$, it will be detected in $T$ phases.
We prove the following key lemma in order to accomplish the proof of the theorem.

\begin{lemma}
Lets assume the underlying matrix $\mathbf{M}$ has row \textit{space non-sparsity number} $k$ and column \textit{space non-sparsity number} $t$. 
Then, if at an intermediate step of $\mathbf{ERRE}$ still column space not recovered completely, then with probability $1-e^{-T \frac{kt}{m}}$ new independent column will be detected within $T$ phases.
\end{lemma}

\begin{proof}
For every active column observation, the probability of detecting independence is at least $1-\frac{k}{m}$ from the lemma \ref{lemma:prob_bound}. 
From the lemma~\ref{lem:activecolumns} , if there is one active column, then there is at least $t$ many active column in that phase.
Therefore, the probability of detection of an active column is at least
$$\Big(1-\frac{k}{m}\Big)^t.$$
Then, we conclude that after $T$ many phase, detection probability is at least
$$\Big(1-\frac{k}{m}\Big)^{t T}.$$ 
Using the inequality $1+ x \leq e^x$ for $\forall x \in \mathbb{R}$ the quantity above can be bounded by:
$$\Big(1-\frac{k}{m}\Big)^{t T} < e^{-T\frac{kt}{m} }.$$
\end{proof}
\hspace{-6mm}
In the rest of the proof we show that with probability at least $1-e^{-T\frac{kt}{m}}$, estimated rank $\widehat{r}$ is equal to $r$.
$$P(r=\widehat{r}) = 1- \big( P(r<\widehat{r}) + P(r>\widehat{r})\big)$$
$P(r<\widehat{r}) = 0$ trivially satisfied, $\widehat{r}$ represents number of detected linearly independent columns of $\mathbf{M}$ which is always bounded by $r$.
Now, all we need to do is to bound $P(r<\widehat{r}) $.
We denote the event of existence of active column by $ACE$. Then, trivially: 
\begin{align*}
  P(\widehat{r} < r) = P\big(\widehat{r} < r \hspace{1mm} \text{and} \hspace{1mm} ACE\big)
\end{align*}
Moreover, we can write 
\begin{align*}
  P\big(\widehat{r}< r \hspace{1mm} \text{and} \hspace{1mm} ACE\big) =\sum_{i=0}^{r-1} P\big(\widehat{r}= i \hspace{1mm} \text{and} \hspace{1mm} ACE\big)= \\
 P\big( ACE  \hspace{1mm} | \hspace{1mm} \widehat{r}= i \big) P(\widehat{r}=i)   
\end{align*}
To finish the proof we just need to observe following equality / inequality :
\begin{align*}
P(\widehat{r} < r) &= \sum_{i=0}^{r-1} P\big(\widehat{r} = i  \wedge  ACE \big)  \\
&= \sum_{i=0}^{r-1} P\big(ACE  |  \widehat{r} = i \big) P(\widehat{r}=i)
\end{align*}
From the lemma above, we can imply that  $$P\big(ACE  |  \widehat{r} = i \big) \leq e^{-T\frac{kt}{m}}$$ and as probability of $P(\widehat{r}=r)\neq 0$, we conclude $P(\widehat{r}<r) < 1$. Equivalently,
$$\sum_{i=0}^{r-1} P\big(\widehat{r}=i\big) < 1$$ which all together these two inequalities concludes 
\begin{align*}
P(\widehat{r} < r) &=\sum_{i=0}^{r-1} P\big(ACE  |  \widehat{r} = i \big) P(\widehat{r}=i)\\ 
&\leq e^{-T\frac{kt}{m}}  \sum_{i=0}^{r-1} P(\widehat{r}=i) \leq e^{-T\frac{kt}{m}}
\end{align*}
To finalize the proof, we divide the algorithm $\mathbf{ERRE}$ into two parts.
First part, is the detection point of the last independent column by algorithm, and second part is waiting T many rounds to check if there is any independent column left.
Moreover, $\mathbf{ERRE}$ would fail generating correct matrix only if failure in the second part happens (there is still independent column not detected, but checking tells us that there is no left) i.e. $\widehat{r} < r$ which we just show $P(\widehat{r} < r) < e^{-T\frac{kt}{m}}$.
This concludes that with probability at least $1-e^{-T\frac{kt}{m}}$ recovered matrix is correct.\medskip \\
Therefore, with probability $1-e^{-T\frac{kt}{m}}$ the first part of the algorithm is just equivalent to the algorithm \hyperlink{err}{$\mathbf{ERR}$}, which with probability more than $1-\epsilon$
observation complexity is bounded by 
\begin{align*}
  (m+n-r)r +  \mathrm{min} \Big(  2 \frac{m n}{k}\log{(\frac{4r}{\epsilon})} , \frac{\frac{2m}{k}(r+2 +\log{\frac{1}{\epsilon}})}{t}n  \Big)     
\end{align*}
using the union bound we conclude that with probability at least $1-\epsilon+e^{-T\frac{kt}{m}}$, the algorithm recovers underlying matrix correctly and observation complexity is bounded by
\begin{align*}
  &(m+n-r)r + Tn + \\
  \mathrm{min} \Big(  2 &\frac{m n}{k}\log{(\frac{4r}{\epsilon})} , \frac{\frac{2m}{k}(r+2 +\log{\frac{1}{\epsilon}})}{t}n  \Big)     
\end{align*}
\end{proof}

\subsection{Proof of Corollary \ref{thm:erre} }

\begin{proof}
We first observe following inequality as implication of given conditions:
\begin{align*} 
    kt \geq m  \implies   \frac{kt}{m} \geq 1 \implies e^{\frac{kt}{m}}\geq e \implies  e^{-\frac{kt}{m}}\leq \frac{1}{e}
\end{align*}
The rest of the proof is just straightforward application of the theorem. 
Setting $T = \log{\frac{1}{\epsilon}}$ to the statement of theorem tells with probability at least $1- ( \epsilon + e^{-\frac{kt}{m} \log{\frac{1}{\epsilon}}})$ using 
\begin{align*}
  &(m+n-r)r+ n\log{\frac{1}{\epsilon}} + \\
  \mathrm{min} \Big(  &2 \frac{m n}{k}\log{(\frac{4r}{\epsilon})} , \frac{\frac{2m}{k}(r+2 +\log{\frac{1}{\epsilon}})}{t}n  \Big)     
\end{align*}
observations.
Considering the fact that $e^{-\frac{kt}{m}}\leq \frac{1}{e}$ we conclude  that 
\begin{align*}
 e^{-\frac{kt}{m} \log{\frac{1}{\epsilon}}}\leq  
\Big( \frac{1}{e} \Big)^{\log{\frac{1}{\epsilon}}} 
= e^{-\log{\frac{1}{\epsilon}}} = e^{\log{\epsilon}} = \epsilon
\end{align*}
Which concludes that with probability at least $1-2\epsilon$ the observation complexity is  bounded by 
\begin{align*}
  &(m+n-r)r+ n\log{\frac{1}{\epsilon}} + \\
  \mathrm{ min} \Big(  &2 \frac{m n}{k}\log{(\frac{4r}{\epsilon})} , \frac{\frac{2m}{k}(r+2 +\log{\frac{1}{\epsilon}})}{t}n  \Big)     
\end{align*}
\end{proof}

\end{document}